\documentclass[12pt]{colt2024} 

\usepackage{algorithm}
\usepackage{overpic}
\usepackage{bbm}
\usepackage{thmtools}

\newcommand{\dd}{\mathrm{d}}

\newcommand{\EX}{\mathbb{E}}
\newcommand{\lmin}{\lambda_{\min,t}}

\DeclareMathOperator{\Tr}{Tr}
\DeclareMathOperator*{\argmin}{\arg\!\min}

\newcommand{\linUCB}{\textnormal{\textsf{LinUCB}}}
\newcommand{\linUCBvn}{\textnormal{\textsf{LinUCB-VN}}}

\newenvironment{skproof}{%
  \renewcommand{\proofname}{Proof sketch.}\proof}{\endproof}

\title[ Linear bandits with polylogarithmic minimax regret]{Linear bandits with polylogarithmic minimax regret}
\usepackage{times}

\coltauthor{%
 \Name{Josep Lumbreras} \Email{josep.lumbreras@u.nus.edu}\\
 \addr Centre for Quantum Technologies, National University of Singapore, Singapore
 \AND
 \Name{Marco Tomamichel} \Email{marco.tomamichel@nus.edu.sg}\\
 \addr Department of Electrical and Computer Engineering and Centre for Quantum Technologies, National University of Singapore, Singapore
}

\begin{document}

\maketitle

\begin{abstract}%
   We study a noise model for linear stochastic bandits for which the subgaussian noise parameter vanishes linearly as we select actions on the unit sphere closer and closer to the unknown vector. 
   We introduce an algorithm for this problem that exhibits a minimax regret scaling as $\log^3(T)$ in the time horizon $T$, in stark contrast the square root scaling of this regret for typical bandit algorithms. Our strategy, based on weighted least-squares estimation, achieves the eigenvalue relation  $\lambda_{\min} ( V_t ) = \Omega (\sqrt{\lambda_{\max}(V_t ) })$ for the design matrix $V_t$ at each time step $t$ through geometrical arguments that are independent of the noise model and might be of independent interest. This allows us to tightly control the expected regret in each time step to be of the order $O(\frac1{t})$, leading to the logarithmic scaling of the cumulative regret.
\end{abstract}


\section{Introduction}

\textbf{Background}. The linear bandit problem~(see, e.g.,~\cite{lattimore_szepesvári_2020}) is one of the most versatile supervised learning frameworks with both applications and strong theoretical guarantees. In this setting, the learner interacts with an environment sequentially, and at each time step $t$ selects an action from a set $\mathcal{A}_t \subset \mathbb{R}^d$ and receives a reward, $r_t$. The reward is a random variable that can be decomposed as $r_t = \EX [r_t] + \epsilon_t$, where $\epsilon_t$ is statistical noise (with zero mean) and the expectation, $\EX [r_t] = \langle \theta, a_t \rangle$, is given by the inner product of the action $a_t \in \mathcal{A}_t$ with an unknown vector $\theta$ from a set of environments $\mathcal{E} \subset \mathbb{R}^d$. The goal of the learner is to minimize the cumulative regret over $T$ rounds, where the regret is the difference  $\max_{a \in \mathcal{A}_t} \langle \theta, a \rangle - \EX [r_t]$, for all possible vectors $\theta \in \mathcal{E}$, i.e., the minimax regret. Here, we focus on the case where both the sets $\mathcal{A}_t$ and $\mathcal{E}$ are continuous.
The challenge posed by linear stochastic bandits lies in efficiently learning the optimal action for $\theta$ while simultaneously choosing actions that minimize the accumulated regret, an intriguing trade-off that our work aims to bring into a new regime.

Numerous algorithms have been proposed to address these challenges for the case of continuous action sets, most prominently \linUCB{} (see~\cite{lin1,rusmevichientong2010linearly,lin3}) and Thompson sampling (see~\cite{abeille2017linear}). Both of these algorithms share the feature that they achieve a minimax regret scaling as $\tilde{O}(\sqrt{T})$, where we are neglecting logarithmic terms and ignoring the dependence on $d$. In fact, under the prevalent assumption that $\epsilon_t$ follows a $\sigma$-subgaussian distribution for some finite $\sigma > 0$, this scaling can be shown to be optimal for various choices of action and environment sets (see~\cite{lin1,shamir2015complexity,rusmevichientong2010linearly}), where the latter work in particular covers the case of unit ball that is of interest in our work. Bandits with heteroscedastic noise have also been studied (see~\cite{kirschner2018information}), where the learner has access to a link function that maps each action to the subgaussian parameter of the noise. The main difference with our model is that our noise also depends on the unknown vector.

\smallskip

\noindent
\textbf{Our model and main result}. In this paper, we present the first non-trivial noise model and corresponding algorithm that achieves a polylogarithmic regret scaling in the time horizon $T$ for linear bandits. We study a stochastic linear bandit with action and environment sets $\mathcal{A}_t = \mathcal{E} = \mathbb{S}^d = \{ x\in\mathbb{R}^d : \| x \|_2 = 1 \}$ and reward model given by $r_t = \langle \theta , a_t \rangle + \epsilon_t$ where $\epsilon_t$ is conditionally $\sigma_t$-subgaussian with
\begin{align}\label{eq:subgaussian_intro}
    \sigma^2_t \leq 1 - \langle \theta ,a_t \rangle^2 \approx \big\| \theta - a_t \big\|_2^2 \, ,
\end{align}
where the approximation holds when $\langle \theta ,a_t \rangle$ is close to 1.

 The proposed model using a noise parameter growing with the distance between unknown vector and action is natural in some contexts. For recommendation systems, for example, it is natural to assume that the choice of a user becomes more certain when the recommendation (action) is close to the user's preference (unknown vector), and the reward may even become deterministic when the recommendation fits the preference perfectly. However our original motivation for the noise scaling comes from a very concrete problem in quantum information, specifically regret-optimal tomography of an unknown pure quantum state (\cite{lumbreras22bandit}). In quantum mechanics measurement outcomes are random with probabilities determined by Born's rule. However, the variance of probabilistic measurement outcomes decreases quadratically for projections aligned with the unknown pure state. 
 
 Beyond specific applications, we believe that a model with decaying noise should be of independent interest as it allows to break the square root minimax regret barrier for continuous action sets, and we believe this is the first model that does this. 

Analytically, this noise model is interesting because the usual methods for deriving lower bounds on the minimax regret fail. First, bounds where the unknown vector $\theta$ is taken from the unit ball as in~\cite{rusmevichientong2010linearly} or~\cite[Chapter 24, Theorem 24.2]{lattimore_szepesvári_2020} do not apply since the worst-case unknown vectors they construct are far from the surface of the ball. However, for finite and constant $\sigma$-subgaussian noise we provide an adaptation of the lower bound for logistic bandits given in~\cite{abeille2021instance} and we obtain a minimax regret scaling of $\Omega (\sigma d \sqrt{T} )$ in Appendix~\ref{ap:lowerbound}. Nonetheless, these arguments fail for the noise model in~\eqref{eq:subgaussian_intro} since they rely crucially on the fact that the distributions induced by two close unknown states have large overlap, which is no longer guaranteed when the noise vanishes. We discuss this in more detail in Appendix~\ref{ap:failurelower}. The lack of applicable lower bounds opens the possibility that we can find algorithms that break the $\Omega(\sqrt{T})$ barrier for the regret. 

Indeed, our main result is an algorithm \linUCBvn{} (see Algoirthm~\ref{alg:weighted_linUCB}), based on \linUCB{} but with a weighted least square estimator that adapts to the vanishing noise. Our algorithm achieves a polylogarithmic scaling of the minimax regret. The main result can be stated as follows, and holds in the model prescribed above. We give the precise statement of the result in Section~\ref{sec:regret_analysis}.
\begin{theorem}[main result, informal version]
 For any $T \in \mathbb{N}$ there exists an instance of \linUCBvn{} such that, for any $\theta \in \mathcal{E}$, we have
    \begin{align}
        \EX[ \textup{Regret}(T) ] = O\left( d^{4} \log^3 (T) \right) .
    \end{align}  
\end{theorem}

The common linear bandit model describes a well-tuned trade-off between exploration and exploitation. Our breaking of the square root regret barrier, which we also observe numerically (see Figure~\ref{fig:numerical_results}), is due to the fact that exploitation can dominate in our model. This is due to the fact that playing actions near the unknown parameter reduces the statistical noise and we will see that we need to explore only the local neighborhood of our current estimate. One very rough way to think about our algorithm is that it tries to select actions as if it were trying to tune into an analog radio station. When we are tuning into a radio station we do not just randomly turn the knob but once we have locked to a signal (the noise starts to decrease) we only need to do small rotations (local exploration of the algorithm) to correct and find it exactly in a few steps.  

\smallskip

\noindent
\textbf{Technical challenges.}  Our approach is based on an optimistic strategy like \textsf{LinUCB}, however, the standard regret analysis fails to give tight bounds on the regret. More precisely, if $\text{reg}_t$ is the instantaneous regret, the usual technique for optimistic strategies used in the literature to bound the cumulative regret is to use the Cauchy–Schwarz inequality, i.e., $\text{Regret}(T) = \sum_{t=1}^T \text{reg}_t \leq ( T\sum_{t=1}^T \text{reg}_t^2 )^{\frac12}$ and then assert that $\sum_{t=1}^T \text{reg}_t^2 = O(\log (T)) $ using the elliptical potential lemma. 
This procedure always introduces a term $\sqrt{T}$ and for that reason, we need to develop a new technique to upper bound the cumulative regret based on a tight control of the instantaneous regret. 

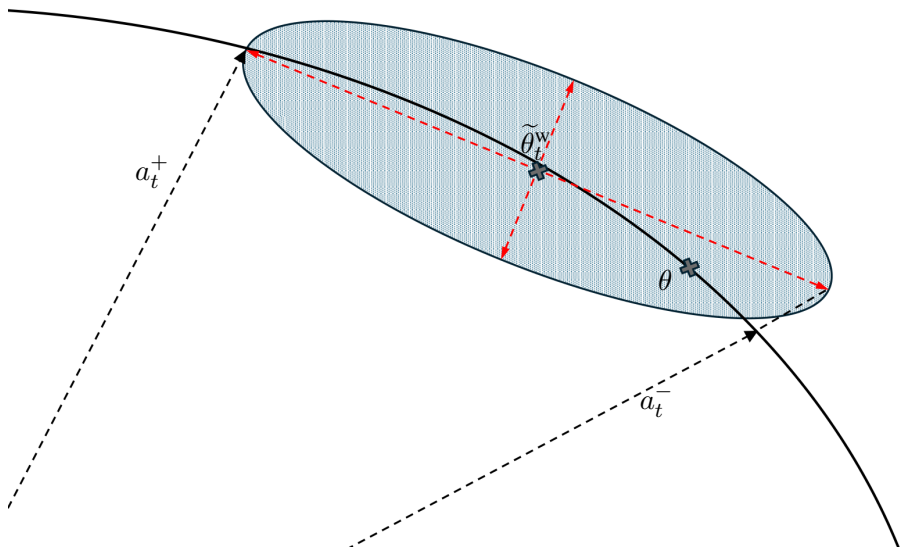
\begin{figure}
\centering

\begin{overpic}[percent,width=0.8\textwidth]{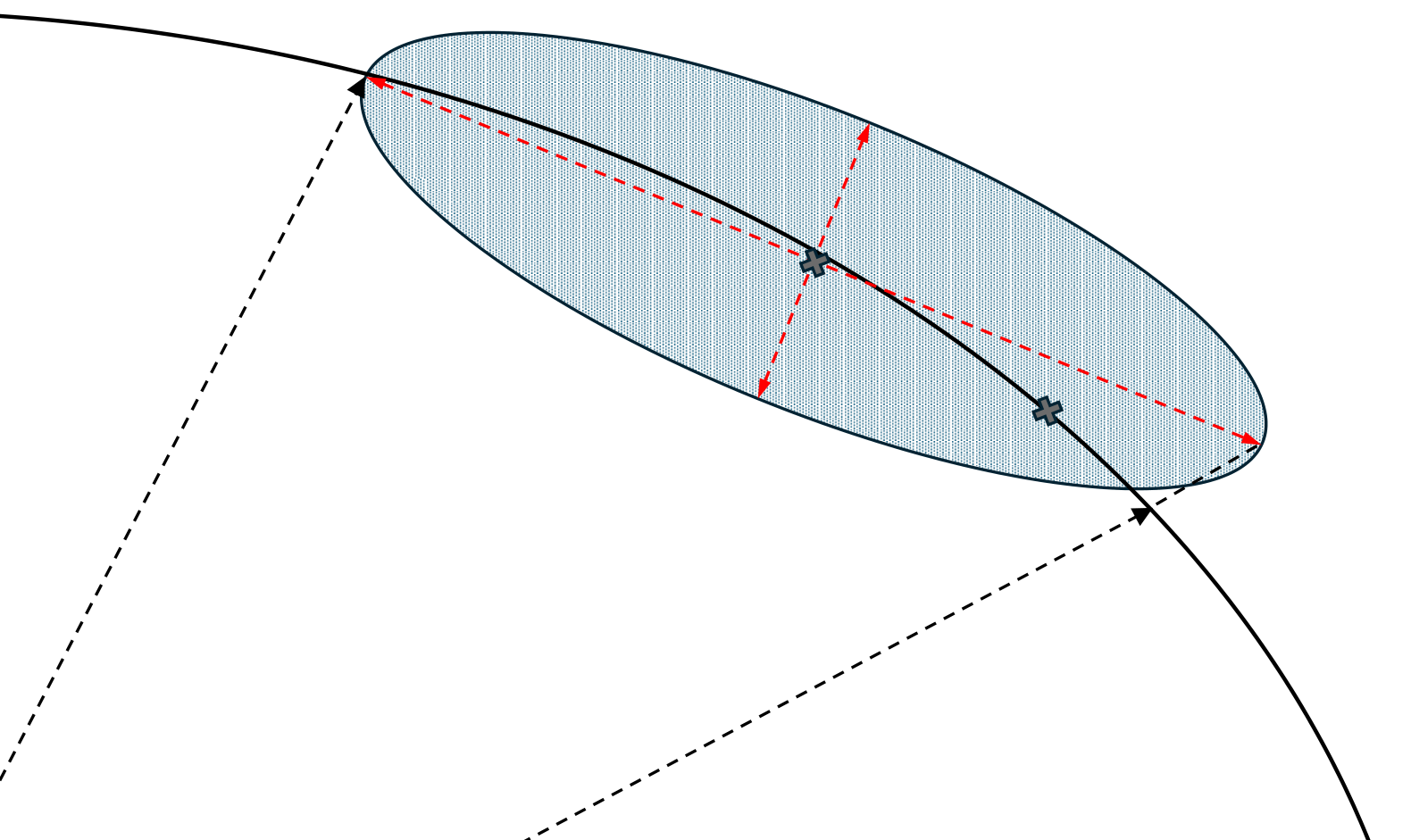}
\put(13,40){ $a^+_t$}
\put(68,15){ $a^-_t$}
\put(70,28){ $\theta$}
\put(55,43){ $\widetilde{\theta}^{\text{w}}_t$}
\end{overpic}
\caption{Scheme for the choice of actions $a_t^+,a_t^-$ of \linUCBvn{}. The actions are selected as the projections of the extremal points across the largest axis of the confidence region centered around a weighted least squares estimator $\widetilde{\theta}^{\text{w}}_t$ of the unknown parameter $\theta$. This choice is sufficient to increase the minimum eigenvalue of $V_t$ such that the relation $\lambda_{\min}(V_t) = \Omega (\sqrt{\lambda_{\max}(V_t)})$ is satisfied. Moreover, the actions $a_t^+$ and $a_t^-$ are sufficiently close to $\theta$ to keep the regret small.}
\label{fig:scheme2}
\end{figure}

Our approach is thus to bound the instantaneous regret $\text{reg}_t$ with high probability for all $t \in [T]$. The algorithm that we design satisfies $\text{reg}_t \sim \| \theta - a_t \|_2^2 \sim 1/\lambda_{\min}(V_t )$  where $V_t$ is the design matrix $V_t = \sum_{s=1}^t w_s a_s a^\mathsf{T}_s$ and $w_s$ is a weight associated to $a_s$. Thus, our regret analysis will require to control the growth of $\lambda_{\min}(V_t)$. From~\cite{banerjee2023exploration} we take the idea that for smooth action sets (in particular the unit sphere) all strategies that minimize the regret must achieve $\lambda_{\min} ( \EX [ V_t ] ) = \Omega ( \sqrt{t})  $ and generalize it to an equivalent condition in terms of eigenvalues of $V_t$ that is noise independent. More precisely, we come up with a geometric construction that ensures that $\lambda_{\min}(V_t) = \Omega ( \sqrt{\lambda_{\max}(V_t)} )$ by choosing actions that explore the confidence region in the direction of the current estimate in a structured way
(see Figure~\ref{fig:scheme2}). 
We then prove that when applying this construction to our noise model we can achieve $\lambda_{\min} ( V_t) = \Omega ( t ) $, which yields a polylogarithmic cumulative regret since $\text{reg}_t \sim \| \theta - a_t \|_2^2 \sim 1/t$.



\section{Notation and model}\label{sec:model}

We present the notation that we will use throughout the paper. For $k\in\mathbb{N}$ we use $[k ] = \lbrace 1,...,k \rbrace$. The set of strictly positive numbers is denoted as $\mathbb{R}_{> 0} = \lbrace x\in\mathbb{R}: x > 0 \rbrace$ and the set of positive semi-definite matrices as $\text{P}^d_+ = \lbrace X\in\mathbb{R}^{d\times d}: X\geq 0 \rbrace$.  For two real vectors $u,v\in\mathbb{R}^d$ we denote its inner product as $\langle u ,v \rangle = \sum_{i=1}^d u_iv_i$ and the Euclidean norm as $\|x \|_2^2 = \langle x , x \rangle$. The $d$-dimensional unit sphere is $\mathbb{S}^d =  \lbrace x\in\mathbb{R}^{d}: \| x \|_2 = 1 \rbrace$ and $\mathbb{B}_r^d (c) = \lbrace x\in\mathbb{R}^d : \| x- c \|_2 \leq r \rbrace$ is the ball of radius $r>0$ centered around $c\in\mathbb{R}^d$. For a vector $x\in\mathbb{R}^d$ and semidefinite positive matrix $A\geq 0$ we denote the weighted norm of $x$ by $A$ as $\| x \|^2_A = \langle x , A x \rangle$. Given an Hermitian matrix $X\in\mathbb{R}^{d\times d}$ we denote $\lambda_{\max}(X)$ and $\lambda_{\min}(X)$ its maximum and minimum eigenvalues, respectively, and by $\lambda_i ( X) $ its $i$-th smallest eigenvalue such that $\lambda_{\min} ( X) =  \lambda_{1}(X )\leq \lambda_2 (X) \leq \cdots \leq \lambda_{d}(X) = \lambda_{\max} (X) $.

We consider a stochastic linear bandit with unknown parameter $\theta\in\mathbb{S}^d$ and action set $\mathcal{A}_t = \mathcal{A} = \mathbb{S}^d$. The learner interacts with the environment over a time horizon of length $[T]$. At time step $t\in[T]$ the learner selects an action $a_t\in\mathcal{A}$ and samples a reward 
\begin{align}\label{eq:quantumclassical_reward}
    r_t = \langle \theta , a_t \rangle + \epsilon_t,
\end{align}
where $\epsilon_t$ is $\sigma_t$-subgaussian, i.e., $\EX [\epsilon_t |\mathcal{F}_{t-1}] = 0$ and $\EX[\exp(\lambda\epsilon_t)|\mathcal{F}_{t-1}] \leq \exp (\lambda^2\sigma^2_t/2)$ where $\mathcal{F}_{t-1} = \sigma\lbrace a_1,r_1,...,a_{t-1},r_{t-1},a_t \rbrace$ is the filtration of all information up to time step $t$ before reward $r_t$ is observed. We consider a noise model such that the subgaussian parameter satisfies
\begin{align}\label{eq:vanishing_noise}
    \sigma_t^2 \leq 1 - \langle \theta , a_t \rangle^2.
\end{align}
We call this bandit a \textit{stochastic linear bandit with linear vanishing noise}. The learner uses a policy $\pi = (\pi_t)_{t=1}^\infty$ where $\pi_t (a_1,r_1,...,a_{t-1},r_{t-1},a_t )$ maps the history up to time step $t$ to a probability distribution over $\mathcal{A}$. The goal of the learner is to minimize the regret defined by
\begin{align}\label{eq:regret}
    \text{Regret} ( T) :=  \sum_{t=1}^T 1 - \langle \theta , a_t \rangle = \frac{1}{2}\sum_{t=1}^T \| \theta - a_t \|_2^2.
\end{align}
where we used that the maximum expected reward is $\max_{a\in\mathcal{A}} \langle a, \theta \rangle = 1$ since $a_t,\theta\in\mathbb{S}^d$ and the last equality follows from $\| x - y \|_2^2 = \langle x - y , x - y \rangle = 2 - 2 \langle x , y \rangle$ for $x,y\in\mathbb{S}^d$.

\section{Weighted regularized least squares estimator and confidence region}\label{sec:lse}

Our strategy is designed to select actions that are close to the unknown parameter and to quantify it we use the confidence region based on the regularized least squares estimator (RLSE) used in the \textsf{LinUCB} algorithm. Moreover, since our noise model is decreasing for actions near the unknown parameter we use a weighted version. In this Section, we review the basic results that we will use later.
The weights of our least squares estimator will be an estimator of the subgaussian parameter $ \hat{\sigma}^2_t$ for the noise $\epsilon_t$. This will allow us to bias our estimator to actions with less noise. We take the idea of using a weighted least squares estimator from~\cite{kirschner2018information} where they apply it to linear bandits with heteroscedastic noise where they can weight the actions with the exact subgaussian parameter. In our setting we need an estimator since the noise depends on the unknown parameter; however, our setting is more concrete and this allows us to design a simpler strategy with theoretical guarantees on the regret.

At time step $t\in[T]$ for action $a_t\in\mathcal{A}$ we define the estimator of the noise as
\begin{align}\label{eq:variance_estimator}
\hat{\sigma}^2_t:\mathcal{H}_{t-1}\times\mathcal{A} \rightarrow \mathbb{R}_{> 0},
\end{align}
where $\mathcal{H}_{t-1} = (a_1,r_1,...,a_{t-1},r_{t-1})$ contains the past information of rewards and actions and it is independent of the unknown parameter $\theta$. To simplify the notation we will often not mention the dependence on the past actions and rewards and simply write $\hat{\sigma}^2_t(a)$.  

The weighted version of the linear least square estimator is defined with the following minimization problem
\begin{align}\label{eq:least_squares}
    \widetilde{\theta}^{\text{w}}_t := \argmin_{\theta' \in\mathbb{R}^d} \sum_{s=1}^t \frac{1}{\hat{\sigma}^2_s(a)}\left(r_s - \langle \theta', a_s\rangle \right)^2 + \lambda \| \theta' \|_2,
\end{align}
where $\lambda\in\mathbb{R}_{>0}$ is a regularization term. The analytical solution is 
\begin{align}\label{eq:estimator_weighted}
    \widetilde{\theta}^{\text{w}}_t  = V_t^{-1} (\lambda) \sum_{s=1}^{t} \frac{1}{\hat{\sigma}^2_s(a)} a_s r_s,
\end{align}
where 
\begin{align}\label{eq:design_matrix}
    V_t (\lambda ) = \lambda I + \sum_{s=1}^{t} \frac{1}{\hat{\sigma}^2_s(a)} a_s a^\mathsf{T}_s,
\end{align}
is the \textit{design matrix}. Using the above estimator for $\theta$ we can define a confidence region under the condition that $\hat{\sigma}^2_t$ is a good estimator for $\sigma_{t}^2$ for any $t$. The Lemma below formalizes this result. The proof is an adaptation of the regularized least squares estimator confidence region in~\cite{lattimore_szepesvári_2020}[Chapter 20] and we include it in the Appendix~\ref{ap:sec3} for completeness. 
\begin{lemma}\label{lem:confidence_region_weighted}
Let $\delta \in (0,1)$ and $(a_t)^\infty_{t=1}$ be the actions selected by some policy with corresponding rewards $(r_t)^\infty_{t=1}$ given  by $r_t = \langle \theta ,a_t \rangle + \epsilon_t$, where $\theta\in\mathbb{S}^d$ is the unknown parameter and $\epsilon_t$ is $\sigma_t$-subgaussian. Let $\hat{\sigma}^2_t$ be an estimator of the form~\eqref{eq:variance_estimator} and define the following event,
\begin{align}\label{eq:gt_event}
  G_t := \left\{ \big( (r_s, a_s)_{s=1}^{t-1}, a_t \big)  :  \sigma^2_s(a_s) \leq \hat{\sigma}^2_s(a_1,r_1,\ldots,a_{s-1},r_{s-1},a_s)\ \forall {s \in [t]} \right\}.
\end{align}
Then we can define the following confidence region 
\begin{align}\label{eq:confidence_region}
    \mathcal{C}_t := \lbrace \theta'\in\mathbb{R}^d : \| \theta' - \widetilde{\theta}^{\textup{w}}_t \|^2_{V_t (\lambda )} \leq \beta_{t,\delta}  \rbrace,
\end{align}
where $V_t$ and $\theta^{\textup{w}}_t$ are defined with $\hat{\sigma}^2_s$  and
\begin{align}\label{eq:beta}
\beta_{t,\delta} = \left( \sqrt{2\log \frac{1}{\delta} + \log\left(\frac{\det(V_t (\lambda ))}{\det (V_0 (\lambda ) )} \right)}+ \sqrt{\lambda}\right)^2. 
\end{align}
Then
\begin{align}\label{eq:prob_confidence}
    \mathrm{Pr}\left[ \forall s \in [t]:  \theta \in \mathcal{C}_s^{\textup{\textup{wls}}} \big| G_t \right] \geq 1 - \delta.
\end{align}
\end{lemma}

\section{Algorithm for linear bandits with vanishing noise: \textbf{LinUCB-VN}}\label{sec:algorithm}

In this Section, we give the specific algorithm that minimizes the regret for the stochastic linear bandit with linear vanishing noise described in Section~\ref{sec:model}. The algorithm is based on the principle of "optimism in the face of uncertainty" (OFU) or upper confidence bounds (UCB). We name the algorithm \textsf{LinUCB-VN}, where VN stands for vanishing noise. The idea is to design a strategy that minimizes the regret~\eqref{eq:regret} which means that we have to select actions close to $\theta$ keeping the trade-off of exploration and exploitation. The \textsf{LinUCB-VN} algorithm is designed to keep the relation $\lambda_{\min} ( V_t ) = \Omega (\sqrt{\lambda_{\max}(V_t)} ) $ at each time step $t$ in line with the result of~\cite{banerjee2023exploration} wich states that all strategies that minimize regret must achieve $\lambda_{\min}(V_t) = \Omega (\sqrt{t})$. This will allow us to bound the regret at each time step $t$ without the need of using the Cauchy-Schwartz inequality in combination with the elliptical potential lemma that gives immediately a  $O(\sqrt{T})$ regret (see~\cite[Chapter 19]{lattimore_szepesvári_2020}). Now we state the exact algorithm and in the subsequent sections we discuss the technical results that allow us to analyze the regret scaling.

\begin{figure}[!h]
    \centering
    \includegraphics[scale=0.5]{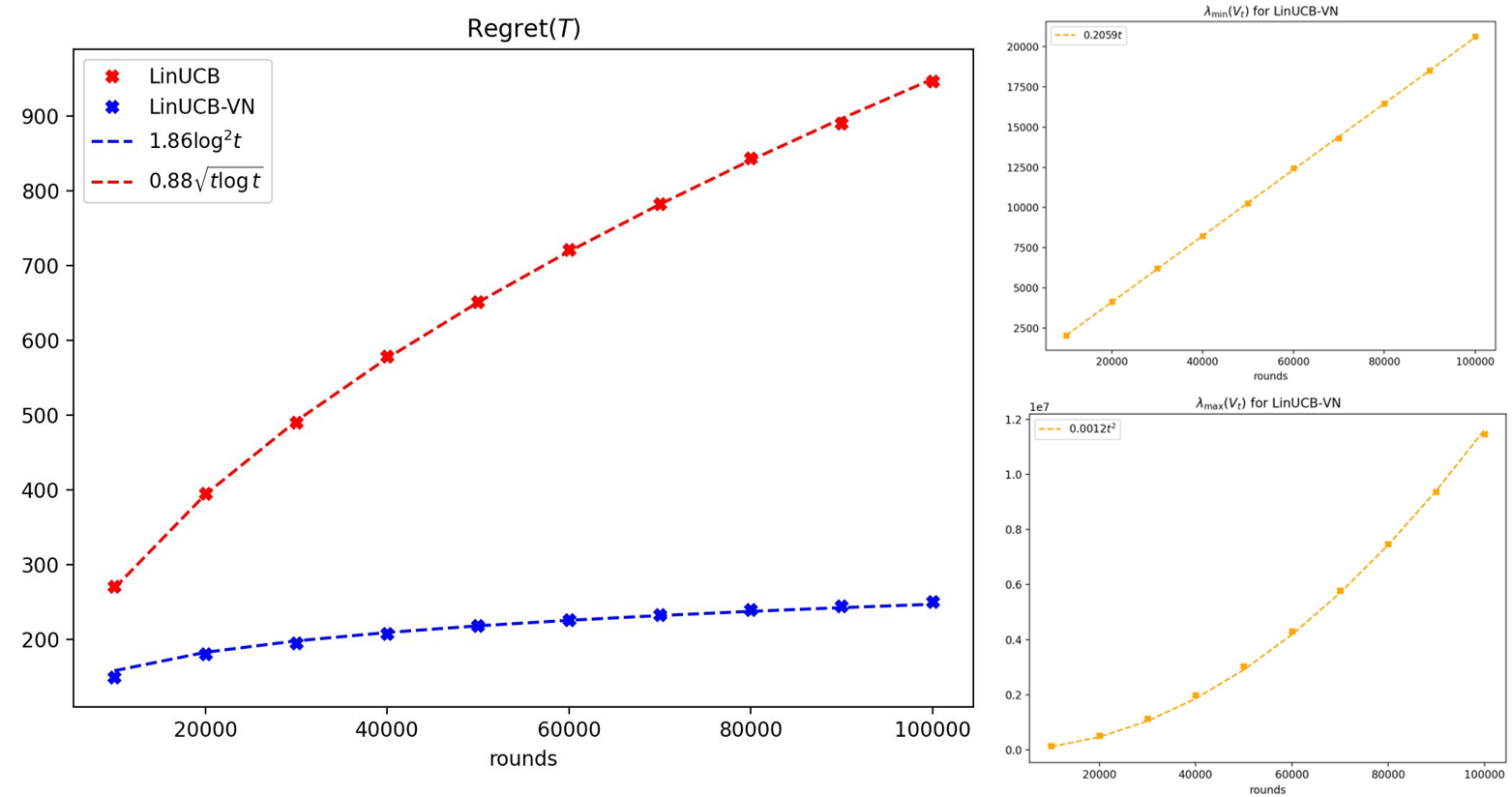}
    \caption{ \small We numerically test \textsf{LinUCB} and \textsf{LinUCB-VN} in a linear bandit with action set $\mathcal{A} = \mathbb{S}^2$ and reward model $r_t = \mathcal{N}(\langle \theta , a_t \rangle , 1 - \langle \theta , a_t \rangle^2 )$. Each point in the graphic is run independently and averaged over 100 instances for random environments $\theta\in\mathbb{S}^2$. \emph{Left plot:} Scaling of the regret for \textsf{LinUCB} algorithm and \textsf{LinUCB-VN}. We fit the functions $R(t) = 1.86\log^2 t$ for \textsf{LinUCB-VN} and $R(t) = 0.88\sqrt{t\log t}$ for \textsf{LinUCB}.
     \emph{Right plots:} Scaling of the maximum and minimum eigenvalue of the matrix $V_t$ for \textsf{LinUCB-VN}. The scaling shows the relation $\lambda_{\min} = \Omega ( \sqrt{\lambda_{\max}} )$. We fit the function  $\lambda_{\min}(V_t) = 0.2059t$ for the minimum eigenvalue and $ \lambda_{\max}(V_t) = 0.0012t^2$ for the maximum eigenvalue. The behavior $\lambda_{\min}(V_t) = \Theta ( t ) $ is the one that gives us the theoretical guarantee of polylogarithmic scaling of the regret.}
    \label{fig:numerical_results}
\end{figure}

Our algorithm works in batches of $2(d-1)$ rounds and with a slightly abuse of notation we will label each batch with $t$. At each batch $t \geq 1$ the algorithm selects the following actions, for $i\in[d-1]$:
\begin{align}\label{eq:general_update}
        a^{\pm}_{t,i} := \frac{\widetilde{a}^{\pm}_{t,i}}{\| \widetilde{a}^{\pm}_{t,i} \|_2},  
        \quad \textrm{where} \quad
        \widetilde{a}^{\pm}_{t,i} := {\theta}^w_{t-1} \pm \frac{1}{\sqrt{\lambda_{t-1,1}}}v_{t-1,i } \quad \text{and} \quad  {\theta}^w_{t} : = \frac{ \widetilde{\theta}^w_{t}}{ \| \widetilde{\theta}^w_{t} \|_2} 
\end{align}
is the normalized least squares estimator of $\widetilde{\theta}^w_{t}$ defined in~\eqref{eq:least_squares}, $v_{t-1,i}$ is the normalized eigenvector with eigenvalue $\lambda_{t-1,i} = \lambda_{i} ( V_{t-1}(\lambda ))$. The design matrix $V_t$ is updated at each batch $t$ as
\begin{align}\label{eq:vt_update}
    V_{t}(\lambda) := V_{t-1}(\lambda) + \omega ( V_{t-1}(\lambda))\sum_{i=1}^{d-1} \left(a^+_{t,i}  (a^+_{t,i})^\mathsf{T}  +  a^-_{t,i}  (a^-_{t,i})^\mathsf{T} \right),
\end{align}
with
\begin{align}\label{eq:omega}
    \omega ( V_{t-1} (\lambda) )  = \frac{\sqrt{\lambda_{\max}(V_{t-1}(\lambda))}}{12\sqrt{d-1}\beta_{t-1,\delta}}, \quad \hat{\sigma}^2_t ({a}^{\pm}_{t,i} ) = \frac{1}{\omega ( V_{t-1}(\lambda) ) },
\end{align}
and $\beta_{t,\delta}$ defined as in Lemma~\ref{lem:confidence_region_weighted} with input $V_t$~\eqref{eq:vt_update}. When is clear from the context we will denote $V_t(\lambda)$ simply as $V_t$. We state the pseudo-code of $\textsf{LinUCB-VN}$ below.

\begin{algorithm}[H]
	\caption{\textsf{LinUCB-VN}} 
	\label{alg:weighted_linUCB}
 
        Require: $\lambda_0\in\mathbb{R}_{>0}$, $\omega: \text{P}^d_+ \rightarrow \mathbb{R}_{\geq 0}$
        
        Set initial design matrix $V_0 \gets \lambda_0\mathbb{I}_{d\times d}$ 
        
        Choose initial estimator ${\theta}_0\in\mathbb{S}^d$ for $\theta$ at random 
        
        \For{$t=1,2,\cdots$}{
            \vspace{1mm}
            \textit{Optimistic action selection}
            \vspace{1mm}
            
            \For{$i = 1,2,\cdots d-1$}{    
                Select actions $a^+_{t,i}$ and $a^-_{t,i}$ according to Eq.~\eqref{eq:general_update}
                
                Receive associated rewards $r^+_{t,i}$ and $r^-_{t,i}$
            }
            
            \vspace{1mm}
            \textit{Update estimator of sub-gaussian noise for $a^+_{t,i}$}
            \vspace{1mm}
            
            $\hat{\sigma}^2_t \gets \frac{1}{\omega ( V_{t-1}(\lambda_0 ) ) }$ for $t\geq 2$ or $\hat{\sigma}^2_t \gets 1$ for $t=1$.

            \vspace{1mm}
            \textit{Update design matrix and RLSE}
            \vspace{1mm}
            
            $V_{t}(\lambda_0) \gets V_{t-1}(\lambda_0) + \frac{1}{\hat{\sigma}_t^2} \sum_{i=1}^{d-1} \left(a^+_{t,i}  (a^+_{t,i})^{\mathsf{T}}  +  a^-_{t,i}  (a^-_{t,i})^\mathsf{T} \right)$
            
            $\widetilde{\theta}_t^\text{w} \gets V_t^{-1} ( \lambda_0 )  \sum_{s = 1}^t \frac{1}{\hat{\sigma}_t^2} \sum_{i=1}^{d-1} (a^+_{s,i} r^+_{t,i} + a^-_{s,i} r^-_{t,i} ) $
        }
\end{algorithm}

\section{Actions and eigenvalue analysis of design matrix}\label{sec:actions}
In this Section, we present our first main result which is an analysis that shows that the actions of the \textsf{LinUCB-VN} algorithm satisfy the relation $\lambda_{\min} ( V_t ) = \Omega (\sqrt{\lambda_{\max}(V_t ) } ) $ at each batch $t$. Moreover, our result is very general since our analysis is independent of the randomness of the algorithm and also the noise model. We give the main ideas of the proof and give details in Appendix~\ref{ap:sec5}.
\begin{theorem}\label{th:main}
    Let $d\geq 2$, $\lbrace c_t \rbrace_{t=0}^\infty \subset \mathbb{S}^d$ be a sequence of normalized vectors and $\omega: \textup{P}^d_+ \rightarrow \mathbb{R}_{\geq 0}$ a function such that 
    \begin{align}
        \omega (X) \leq C\sqrt{\lambda_{\max}(X)} \quad \forall X\in \textup{P}^d_+,
    \end{align}
     for some constant $C > 0$. Let $\lambda_0 \geq \max \big\lbrace 2,\sqrt{\frac{2}{3(d-1)}}2dC+\frac{2}{3(d-1)} \big\rbrace$, and define a sequence of matrices
     $\lbrace V_t \rbrace_{t=0}^\infty \subset \mathbb{R}^{d\times d}$ as
       \begin{align}\label{eq:vt_lemma}
         V_0 := \lambda_0\mathbb{I}_{d\times d}, \quad      V_{t+1} := V_t + \omega ( V_t ) \sum_{i=1}^{d-1}P_{t,i}, 
       \end{align}
       where 
       \begin{align}\label{eq:defP_a}
           P_{t,i} : = a^+_{t+1,i}(a^+_{t+1,i})^\mathsf{T}  +  a^-_{t+1,i} (a^-_{t+1,i})^\mathsf{T} , \quad
           a^\pm_{t+1,i} : = \frac{\tilde{a}^\pm_{t+1,i}}{\| \tilde{a}^\pm_{t+1,i}\|_2}, \quad \tilde{a}^\pm_{t+1,i} := c_t \pm \frac{1}{\sqrt{\lambda_{t,1}}} v_{t,i},
       \end{align}
       with $\lambda_{t,i} = \lambda_{i}(V_t)$ the eigenvalues of $V_t$ with corresponding normalized eigenvectors \\$v_{t,1},...,v_{t,d}\in\mathbb{S}^d$.
  Then we have
    \begin{align}\label{eq:eig_relation}
        \lambda_{\min}(V_t) \geq \sqrt{\frac{2}{3(d-1)}\lambda_{\max}(V_t)} \quad \text{for all}\quad t\geq 0.
    \end{align}
  
\end{theorem}

\begin{skproof}
    The idea to prove this result is to use induction and analyse separately the cases when $\lambda_{t,1} < \lambda_{t,d}$ and $\lambda_{t,1} \approx \lambda_{t,d}$ for each $t\geq 1$. The first case is the most delicate and we must quantify the minimum (maximum) non-trivial perturbation that the terms $P_{t,i}$ can contribute to the smallest (biggest) eigenvalue when adding to the matrix $V_t$. 
    For the second case, the idea is that if $\lambda_{t,1} \approx \lambda_{t,d}$ then if the inequality~\eqref{eq:eig_relation} holds at time step $t$, the term $ \sum_{i=1}^{d-1}P_{t,i}$ is not able to break the relation even if it does not contribute to increase the minimum eigenvalue. 
 
    \textbf{Case 1}: $\lambda_{t,1} < \lambda_{t,d}$. For this case, we want to establish a matrix inequality between $V_{t+1}$ and $V_t$ that is independent of the sequence $\lbrace c_t \rbrace_{t=0}^\infty$. A simple calculation shows $\| \tilde{a}^\pm_{t+1,i}\|_2 \leq  1 + 1/\sqrt{\lambda_{t,1}} $ and using that $\tilde{a}^+_{t+1,i}(\tilde{a}^+_{t+1,i})^\mathsf{T}  + \tilde{a}^-_{t+1,i} (\tilde{a}^-_{t+1,i})^\mathsf{T}$ does not depend on the cross terms $c_tv^\mathsf{T}_{t,i}$ we can show
    \begin{align}
        P_{t,i} \geq \frac{2}{(1+\sqrt{\lambda_{t,1}})^2} v_{t,i}v^\mathsf{T}_{t,i},
    \end{align}
    which quantifies how much $P_{t,i}$ contributes to increase the eigenvalue on the direction $v_{t,i}$. This immediately leads to 
    \begin{align}\label{eq:vt+1_bound_sk}
        V_{t+1} \geq V_t +  p_t \sum_{i=1}^{d-1} v_{t,i} v^\mathsf{T}_{t,i}, \quad p_t : = \frac{2 \omega(V_t)}{(1+\sqrt{\lambda_{t,1}})^2}, 
    \end{align}
    which is the matrix inequality that we will use to relate the eigenvalues of $V_{t+1}$ and $V_t$. In our proof we distinguish the cases $\lambda_{t,d-1} + p_t \leq \lambda_{t,d}  $ and $\lambda_{t,d-1}+ p_t < \lambda_{t,d}  < \lambda_{t,1}+ p_t$. However, since both cases use the same technique (with some subtleties) we focus on the first case here and leave the details of the second in the Appendix~\ref{ap:sec5}. Using the inequality~\eqref{eq:vt+1_bound_sk} in combination with the minimax principle for eigenvalues in~\cite[Corollary III.1.2]{bhatia97} we have
    \begin{align}\label{eq:lambdaminbound_sk}
    \lambda_{t+1,i} \geq \lambda_{t,i} + p_t \quad \text{for} \quad i=1,...,d-1 .
    \end{align}
    Then we can control how much the maximum eigenvalue grows using the above bound and $ \Tr (V_{t+1}) = \lambda_{t,d}+\sum_{i=1}^{d-1} \lambda_{t,i}+2(d-1)\omega(V_t)$. This allows us to establish that
    \begin{align}\label{eq:lambda_upper_sk}
    \lambda_{t+1,d} \leq \lambda_{t,d} + (2\omega(V_t ) - p_t)(d-1).
    \end{align}
    We finalize this case showing that
    \begin{align}
        (\lambda_{t,1}+p_t)^2 \geq \frac{2}{3(d-1)}(\lambda_{t,d} + (2\omega(V_t ) - p_t)(d-1) ),
    \end{align}
    under the assumption that induction relation~\eqref{eq:eig_relation} holds at time step $t$. Then the inequality at time step $t+1$ $\lambda_{t+1,1} \geq \sqrt{2/(3(d-1)\lambda_{t+1,d}}$ follows from~\eqref{eq:lambdaminbound_sk} and~\eqref{eq:lambda_upper_sk}.

     \textbf{Case 2}: $\lambda_{t,1} \approx \lambda_{t,d}$. For this case, we do not need a tight control on the eigenvalues of $V_{t+1}$. It suffices to use the following bounds 
    \begin{align}\label{eq:trivial_pert_bound}
        \lambda_{t+1,1} &\geq \lambda_{t,1}, \\
        \lambda_{t+1,d} &\leq \lambda_{t,d} + 2 \omega(V_t ) (d-1) ,
    \end{align}
    which follows from the inequalities  $\lambda_{\min}( A+B) \geq \lambda_{\min}( A ) + \lambda_{\min}( B )$ and \\$\lambda_{\max}(A+B) \leq \lambda_{\max}(A ) + \lambda_{\max}(B)$ for $A,B \in \text{P}^d_+$. Then combining the induction hypothesis at time step $t$,~\eqref{eq:eig_relation} the condition $\omega(V_t) \leq C \sqrt{\lambda_{t,d}}$ and $\lambda_{t,1} \geq \lambda_{0}$, we can establish a series of inequalities that lead to
     \begin{align}
     \lambda_{t+1,1} \geq \lambda_{t,1}  \geq \sqrt{ \frac{2}{3(d-1)}\left(2(d-1)\omega(V_t) + \lambda_{t,d} \right)} \geq \sqrt{ \frac{2}{3(d-1)}\lambda_{t+1,d}},
    \end{align}
    where the first and last inequality follows from the trivial perturbation bounds~\eqref{eq:trivial_pert_bound}. This concludes the induction.
\end{skproof}

\textbf{Note.} In the Appendix~\ref{ap:sec3_case2} we provide an alternative proof for the special case of $d=2$ that slightly improves the constant in the relation $\lambda_{\min}(V_t) = \Omega ( \sqrt{\lambda_{\max} ( V_t ) } ) $. In this proof, we compute the worst-case scenario depending on the centers $c_t$ and we provide an exact computation of the eigenvalues of the matrix $V_t$.

\section{Regret analysis}\label{sec:regret_analysis}

In this Section, we present the regret analysis of \textsf{LinUCB-VN} for a linear bandit with linear vanishing noise. We use the result of the previous section to show a polylogarithmic scaling of the regret with high probability. We develop a new technique to analyze an algorithm based on upper confidence bounds and we rely on the relation that our algorithm keeps the relation  $\lambda_{\min} ( V_t ) = \Omega (\sqrt{\lambda_{\max}(V_t )}  ) $.

\begin{theorem}\label{th:regret_bound_d2}
Let $d\geq 2$, $\delta\in ( 0,1)$ and $T = 2(d-1)\widetilde{T}$ for some $\widetilde{T}\in\mathbb{N}$. Let $\omega (X)$ defined as in~\eqref{eq:omega} using $\delta$ and $\lambda_0$ satisfy the constraints in Theorem~\ref{th:main}. Then if we apply \textsf{LinUCB-VN}~\ref{alg:weighted_linUCB}$(\lambda_0 ,\omega)$ to a $d$ dimensinal stochastic linear bandit with linear vanishing noise~\eqref{eq:vanishing_noise} with probability at least $1-\delta$ the regret satisfies
\begin{align}
       \textup{Regret}(T) \leq 4(d-1) + \left(144d^2\beta^2_{T,\max}+24(d-1)^{\frac{3}{2}}\beta_{T,\max}\right)\log \left( \frac{T}{2(d-1)} \right),
\end{align}
    where
    \begin{align}
        \beta_{T,\max} : = \left(\lambda_0+\sqrt{2\log\frac{T}{\delta} + d\log\left(\frac{1}{144\lambda_0}T^2 + \frac{1}{6\sqrt{\lambda_0}}T +1\right)} \right)^2 .
    \end{align}
\end{theorem}

From the above Theorem we have that the scaling of the regret on $d$ and $T$ is
    \begin{align}
    \textup{Regret}(T) = O(d^{4}\log^3 (T)),
    \end{align}
and in Corollary~\ref{cor:expected_regret_linucbvn} we prove that choosing $\delta = \frac{1}{\widetilde{T}}$ gives  
\begin{align}
\EX [\textup{Regret}(T) ] =  O(d^{4}\log^3 (T)).
\end{align}

\begin{skproof}
  Using that the algorithm works in $\widetilde{T}$ batches of $2(d-1)$ actions we can express the regret as
    \begin{align}
        \text{Regret}(T) =   \frac{1}{2}\sum_{t= 1}^{\widetilde{T}} \sum_{i = 1}^{d-1} \left( \| \theta - a^+_{t,i} \|_2^2 + \| \theta - a^{-}_{t,i} \|_2^2  \right).
    \end{align}
   Thus, to upper bound the regret, we have to give a high probability bound on $\| \theta - a^{\pm}_{t,i} \|_2^2 $ for each batch $t\in[\widetilde{T}]$. For that, we can use triangle inequality to obtain
   \begin{align}\label{eq:theta_at_bound_sk}
    \| \theta - a^\pm_{t,i} \|_2 &\leq \| \theta - \widetilde{\theta}^{\text{w}}_{t-1} \|_2 + \| \theta^{\text{w}}_{t-1}  - \widetilde{\theta}^{\text{w}}_{t-1} \|_2 + \| {\theta}^{\text{w}}_{t-1} -a^\pm_{t,i} \|_2 .
\end{align}
    Under the assumption that $\theta\in\mathcal{C}_{t-1}$ with $\mathcal{C}_{t-1}$ defined as in Lemma~\ref{lem:confidence_region_weighted} we can use that $\mathcal{C}_{t-1} \subseteq \mathbb{B}^d_{r}(\widetilde{\theta}^{\text{w}}_{t-1})$ with $r = \sqrt{\frac{\beta_{t-1,\delta}}{\lambda_{\min}(V_t)}}$ and upper bound the three terms by $r$. The bound on $\| \theta - \widetilde{\theta}^{\text{w}}_{t-1} \|_2 $ follows from the definition of $C_{t-1}$ and the remaining terms we provide geometrical proofs in the Appendix~\ref{ap:sec6}. Now the main idea is to use our previous result Theorem~\ref{th:main} (under a careful check of the assumptions and fixing $\lambda_0$) and use that our particular choice of actions allows us to keep the inequality $\lambda_{\min}(V_t) \geq \sqrt{\frac{2}{3(d-1)} \lambda_{\max}(V_t)}$. Then we can upper bound $\| \theta - a^{\pm}_{t,i} \|_2^2$  as
    \begin{align}\label{eq:dist_lambdmax}
    \| \theta - a^{\pm}_{t,i} \|_2^2 &\leq \frac{9\beta_{t-1,\delta}}{\lambda_{\min(V_{t-1})}} \leq \frac{12\sqrt{d-1}\beta_{t-1,\delta}}{\sqrt{\lambda_{\max}(V_{t-1})}}.
    \end{align}
     Thus, if we can control the growth of $\lambda_{\max} (V_t ) $ then we can bound the regret under the assumption that $\theta\in\mathcal{C}_t$ for all $t\in[\widetilde{T} ] $. In the Appendix~\ref{ap:sec6}, we rigorously prove that in fact, we can apply Lemma~\ref{lem:confidence_region_weighted} and bound the probability that the following event holds
    \begin{align}\label{eq:event_Et_sk}
        E_t := \lbrace \big( r^+_{s,1}, a^+_{s,1},r^-_{s,1}, a^-_{s,1},...,r^+_{s,d-1}, a^+_{s,d-1},r^-_{s,d-1}, a^-_{s,d-1})\big)_{s= 1}^{t} : \forall s \in \left[ t\right], \theta \in \mathcal{C}_{s} \rbrace .
    \end{align}
    To apply Lemma~\ref{lem:confidence_region_weighted} we need that the event $G_t$~\eqref{eq:gt_event} holds which means essentially that $\hat{\sigma}^2_t $~\eqref{eq:omega} is a good estimator of the subgaussian parameter of the noise.
    Using the definition of the subgaussian parameter we have $\sigma^2_{t}(a^\pm_{t,i}) \leq \| \theta - a^\pm_{t,i} \|_2^2.$ and from the above argument we see that we can define a "good" estimator of the noise $\hat{\sigma}^2_t (a^\pm_{t,i} )$ in the sense that
    \begin{align}\label{eq:upperbound_subgaussian_sk}
        \sigma^2_t ( a^\pm_{t,i}) \leq \hat{\sigma}^2_t (a^\pm_{t,i} ) = \frac{1}{w(V_{t-1} )} \quad \text{if} \quad \theta\in\mathcal{C}_{t-1}.
    \end{align}
    Thus, after bounding the probabilities that the events $G_t$ and $E_t$ hold what remains to do is to provide bounds for $\lambda_{\max}(V_{t})$.  Our result will show $\lambda_{\max}(V_{t}) = \Theta (t^2 )$ and we provide a careful computation in the Appendix~\ref{ap:sec6}. To do that, we just need to use the rule update of $V_t$ and our computation will be independent of the algorithm. Here we provide the idea for a lower bound on $\lambda_{\max}(V_{t})$ that is essentially what we need to upper bound the regret. A similar idea works for an upper bound and we provide the full computation in the Appendix~\ref{ap:sec6}. The upper bound is necessary to prove $\beta_t \leq \beta_{T,\max}$. Let's provide the main ideas for the lower bound.

    From the definition of $V_t$ and the choice of $\omega ( V_t )$~\eqref{eq:omega} we have
  \begin{align}
    \Tr (V_t ) &\geq \sum_{s= 2}^{t} 2(d-1)\omega (V_{s-1} ) \geq \frac{\sqrt{d-1}}{6\beta_{\tilde{T},\delta}} \sum_{s=1}^{t-1} \sqrt{\lambda_{\max}(V_s)}. 
\end{align}
   Then we can use $\Tr ( V_s) \leq d\lambda_{\max}(V_s) $ and some algebra to establish the following inequality 
\begin{align}\label{eq:lambdmaxinequality_sk}
    \lambda_{\max}(V_t) \geq \frac{1}{1+6\frac{d}{\sqrt{d-1}}\beta_{\tilde{T},\delta}} \sum_{s= 1}^{t} \sqrt{\lambda_{\max}(V_s)}.
\end{align}
 This inequality already gives the intuition that  $\lambda_{\max}(V_t) = \Omega ( t^2 )$ because if we propose $\lambda_{\max}(V_t) \sim t^\alpha$ and we substitute in the above inequality, a comparison of exponents shows that $\alpha \geq 1+\alpha/2$. Solving this gives $\alpha \geq 2$. However, we have to formalize this idea. To do that we are going to extend the function $\lambda_{\max}(V_t)$ to the continuous on $t$ and this will allow us to establish a differential inequality that we can exactly solve. First, we define the following linear interpolation
\begin{align}\label{eq:def_fx_sk}
    g_2(x) := (t+1-x)\lambda_{\max}(V_t) + (x-t)\lambda_{\max}(V_{t+1}) \quad \text{for} \quad x\in[t,t+1) \quad \text{and} \quad t\in \lbrace 0 , ..., \widetilde{T} \rbrace,
\end{align}
which satisfies $g_2(t) = \lambda_{\max}(V_t) $. Combining the above definition with the inequality~\eqref{eq:lambdmaxinequality_sk} and some algebra we can show that
\begin{align}
    g_2(x) &\geq \frac{1}{1+6\frac{d}{\sqrt{d-1}}\beta_{\tilde{T},\delta}} \int_{0}^x \sqrt{g_2 (s)}ds.
\end{align}
This leads to the following differential inequality  
\begin{align}
    \frac{dG_2(x)}{dx} \geq \sqrt{\frac{1}{1+6\frac{d}{\sqrt{d-1}}\beta_{\tilde{T},\delta}}}\sqrt{G_2(x)}.
\end{align}
where $ \frac{dG_2(x)}{dx} = \sqrt{g_2(x) }$. Using a result of differential equations (more details in the Appendix~\ref{ap:sec6}) we can solve just the equality and show that 
\begin{align}
    G_2 (x) &\geq \frac{1}{4+24\frac{d}{\sqrt{d-1}}\beta_{\tilde{T},\delta}}x^2 .
\end{align}
This leads to the result $\lambda_{\max}(V_t ) = \Omega (t^2)$. Thus, from the expression of the regret, $\beta_{t,\delta} = O(\log(T))$ and~\eqref{eq:dist_lambdmax} we have
\begin{align}
    \text{Regret}(T) = O \left(\text{polylog}(T)\sum_{t=1}^{\widetilde{T}} \frac{1}{\sqrt{\lambda_{\max}(V_t ) }} \right) = O \left(\text{polylog}(T)\sum_{t=1}^{\widetilde{T}} \frac{1}{t} \right) = O \left(\text{polylog}(T) \right).
\end{align}

\end{skproof}
\textbf{Note.} Setting $\omega(V_t ) = 1$ allows the algorithm to deal with the usual 1-subgaussian noise. For this case, our proof simplifies and leads to the usual $\text{Regret}(T) = \widetilde{O}(\sqrt{T})$.
\section{Open problems}
\begin{itemize}
    \item Our model can be extended in different ways, but we felt a focus on a relatively simple geometry would be more suitable to introduce new techniques. A generalisation to Locally Constant Hessian (LCH) surfaces and locally convex action sets~\cite{banerjee2023exploration} would require us to replace the normalization of the actions by a projection to the corresponding surface. We expect that this would still yield a relation $\lambda_{\min} = \Omega (\lambda^s_{\max})$ where $s\in(0,1/2]$ depends on how well the action set approximates a LCH, and from there the analysis would be similar. A relatively straight-forward generalisation of our analysis also works for noise decaying as $\|\theta -a_t \|_2^{2\alpha}$ for $\alpha \in [0,1]$. This eventually leads to a regret scaling as $T^{\beta}\text{polylog}(T)$ with $\beta = \frac{1-\alpha}{2-\alpha}$. This behaviour interpolates between the $\sqrt{T}$ seen in the constant noise case and $\text{polylog(T)}$ seen in our work.
    \item In~\cite{banerjee2023exploration} the authors prove (under some mild assumptions) that any strategy that minimizes the regret for linear bandits with continuous smooth action sets must satisfy $\lambda_{\min}(V_t) = \Omega (\sqrt{t} ) $ for constant noise models. This result does not apply to our model since the noise is not constant. Our strategy achieves the relation $\lambda_{\min}(V_t) = \Omega (\sqrt{\lambda_{\max}(V_t)} ) $ that is a more general condition independent of the noise. Thus, we propose the following conjecture:
    \begin{conjecture}
        (informal) Consider a linear stochastic bandit with $\mathcal{A} \subset \mathbb{R}^d$ a smooth continuous action set and reward model with arbitrary bounded noise. Then, any strategy that minimizes the regret must achieve the relation $\lambda_{\min}(V_t) = \Omega (\sqrt{\lambda_{\max}(V_t)} ) $.
    \end{conjecture}
    \item How can we extend the linear stochastic bandit techniques for minimax lower bound of the regret to the non-constant noise setting? Our model misses a matching minimax lower bound.
    \item Our strategy plays batches of $2(d-1)$ actions. This is done to simplify the calculations of our main Theorems. However, a technique that adapts at each time step could potentially reduce the dimensional dependence of the regret. 
\end{itemize}

\acks{The authors thank Aditya Gopalan for interesting discussions during the Information Theory and Data Science Workshop 2023 at NUS, Singapore. The authors also thank Junwen Yang and Vincent Tan for early discussion about this problem and Yonglong Li for discussions about confidence regions. JL thanks Roberto Rubboli for everyday discussions and Marco Fanizza and Andreas Winter for sharing different perspectives on the problem during JL's visit to GIQ in Barcelona in 2022.

MT thanks the Pauli Center who supported a long-term visit
to ETH Zürich during the initial phase of this project. This research is supported by the National Research Foundation, Singapore and A*STAR under its CQT Bridging Grant and the Quantum Engineering Programme grant NRF2021-QEP2-02-P05.}

\bibliography{biblio_purestatebandit}

\newpage
\appendix


\section{Proofs of Section 3}\label{ap:sec3}

The proof of the confidence region for the regularized least squares estimator relies on the following result about supermatingales.

\begin{theorem}[Theorem 3.9 in~\cite{lattimore_szepesvári_2020}]\label{th:supermartingale}
    Let $(X_t)_{t=0}^\infty$ be a supermartingale with $X_t \geq 0$ almost surely for all $t$. Then for any $\epsilon > 0$,
    \begin{align}
        \mathrm{Pr} \left( \sup_{t} X_t \geq \epsilon \right) \leq \frac{\EX \left[ X_0\right]}{\epsilon}.
    \end{align}
\end{theorem}

\subsection{Proof of Lemma~\ref{lem:confidence_region_weighted}}

The proof below is an adaptation from the one presented in~\cite{lattimore_szepesvári_2020}[Chapter 20]  to our setting. We refer to the reference for detailed computations.

\begin{proof} 
To simplify notation we use $\hat{\sigma}_s^2 = \hat{\sigma}_s^2 (a_s)$. From now we will condition all our calculations on the events $G_t$ since we want to prove~\eqref{eq:prob_confidence}. 

First let $\widetilde{S}_t = \sum_{s=1}^t \frac{1}{\hat{\sigma}_s^2}\epsilon_s a_s$  and define the following process
\begin{align}
    \widetilde{M}_t (x) := \exp \left( \langle x , \widetilde{S}_t \rangle - \frac{1}{2} \| x \|^2_{\widetilde{V}_t} \right),
\end{align}
for all $x\in\mathbb{R}^d$ and $V_t = V_t (0)$. We want to check that $\widetilde{M}_t(x)$ is a supermartingale, i.e $\EX \left[ \widetilde{M}_t (x) | \mathcal{F}_{t-1} \right] \leq \widetilde{M}_{t-1} (x)$. From a direct calculation, we have
\begin{align}\label{eq:supermartingale_check}
    \EX \left[ \widetilde{M}_t (x) | \mathcal{F}_{t-1} \right] = \widetilde{M}_{t-1}(x) \EX \left[ \exp \left( \frac{\epsilon_t}{\hat{\sigma}_t} \left\langle x ,\frac{a_t}{\hat{\sigma}_t}\right\rangle - \frac{1}{2}\| x \|_{\frac{1}{\hat{\sigma}^2_t} a_t a^\top_t} \right) \big| \mathcal{F}_{t-1}\right].
\end{align}
Then, using that in the definition of $\hat{\sigma}^2_t$~\eqref{eq:variance_estimator} is defined only using the information up to time step $t-1$, the subgaussian property, and that the event $G_t$ holds we have
\begin{align}
    \EX \left[ \exp \left( \frac{\epsilon_t}{\hat{\sigma}_t} \left\langle x ,\frac{a_t}{\hat{\sigma}_t}\right\rangle \right) \big| \mathcal{F}_{t-1}\right] \leq \exp \left(  \frac{1}{2}\left\langle x, \frac{a_t}{\hat{\sigma}_t} \right\rangle^2 \right) = \exp \left( \frac{1}{2}\| x \|_{\frac{1}{\hat{\sigma}^2_t} a_t a^\top_t}\right).
\end{align}
Inserting the above expression into~\eqref{eq:supermartingale_check} we immediately get $\EX \left[ \widetilde{M}_t (x) | \mathcal{F}_{t-1} \right] \leq \widetilde{M}_{t-1} (x)$. Using Lemma 20.3 in~\cite{lattimore_szepesvári_2020}[Chapter 20] we have that
\begin{align}
    \Bar{M}_t := \int_{\mathbb{R}^d} \widetilde{M}_t (x) d h(x)
\end{align}
is a supermartingale where $h$ is a probability measure on $\mathbb{R}^d$. In particular we choose $h = \mathcal{N} ( 0 , H^{-1} )$ with $H = \lambda \mathbb{I}_{d\times d} \in \mathbb{R}^{d \times d}$ and we get
\begin{align}
\Bar{M}_t = \frac{1}{\sqrt{(2\pi)^d \det (H^{-1} )}} \int_{\mathbb{R}^d} \exp \left( \langle x , \widetilde{S}_t \rangle - \frac{1}{2} \| x\|^2_{V_t} - \frac{1}{2} \| x \|_{H} \right)dx .
\end{align}
The above quantity can be exactly computed as
\begin{align}\label{eq:expected_mt}
    \Bar{M}_t = \left( \frac{\det (V_0 )}{\det  ( V_t(\lambda) )}\right)^{\frac{1}{2}}\exp \left( \frac{1}{2} \| \widetilde{S}_t \|^2_{V^{-1}_t(\lambda )} \right),
\end{align}
where we used the solution for a Gaussian integral, we have completed the square inside the exponential term and we used that $V_t(\lambda ) = H + V_t$. Then we can use Theorem~\ref{th:supermartingale} in order to get
\begin{align}
    \mathrm{Pr} \left ( \sup_{t} \log \left( \Bar{M}_t\right) \geq \log \left( \frac{1}{\delta} \right) \right) = \mathrm{Pr} \left( \sup_t \Bar{M}_t \geq \frac{1}{\delta} \right) \leq \delta ,
\end{align}
where we used that by definition of $\widetilde{M}_{t-1}(x)$ we have $\Bar{M}_t \geq 0$ almost surely and $\Bar{M}_0 = 1$. Inserting~\eqref{eq:expected_mt} into the above equation we have
\begin{align}\label{eq:prob_wt}
    \mathrm{Pr}\left( \sup_{t} \| \widetilde{S}_t \|^2_{V^{-1}_t(\lambda )} \geq 2 \log \left( \frac{1}{\delta}\right) + \log \frac{\det  ( V_t(\lambda) )}{\det (V_0(\lambda))} \right) \leq \delta .
\end{align}
Finally using the expression for the weighted least squares estimator~\eqref{eq:estimator_weighted} we have
\begin{align}
    \| \tilde{\theta}^{\text{wls}}_t - \theta \|_{V_t (\lambda )} \leq \| \widetilde{S}_t \|_{V^{-1}_t(\lambda )} + \sqrt{\lambda},
\end{align}
where we used triangle inequality and $\|\theta \|_2^2=1$. And the result follows by combining the above expression with~\eqref{eq:prob_wt} and conditioning under the event $G_t$.
\end{proof}

\section{Proofs of Section 5}\label{ap:sec5}

We will need the following properties for positive semidefinite matrices $A,B\in \text{P}^d_+ $:
  
\begin{align}\label{eq:pertboundsAB}
       \lambda_{\min}( A+B) &\geq \lambda_{\min}( A ) + \lambda_{\min}( B ), \\
       \lambda_{\max}(A+B) &\leq \lambda_{\max}(A ) + \lambda_{\max}(B). \nonumber
\end{align}
And the following mini-max characterization of eigenvalues for Hermitian matrices.

\begin{corollary}[ Corollary III.1.2 in~\cite{bhatia97} ]\label{cor:minimax_eigenvalues}
    Let $A \in \mathbb{C}^{d\times d}$ be a Hermitian matrix, then
    \begin{align}
        \lambda_{k} (A ) = \max_{\substack{\mathcal{M} \subset \mathbb{C}^d \\ \dim ( \mathcal{M} ) = d-k+1}} \min_{\substack{x\in\mathcal{M} \\ \| x\|_2 = 1}} \langle x , A x \rangle = \min_{\substack{\mathcal{M} \subset \mathbb{C}^d \\ \dim ( \mathcal{M} ) = k}} \max_{\substack{x\in\mathcal{M} \\ \| x\|_2 = 1}} \langle x , A x \rangle.
    \end{align}
    In particular, if $A \geq B$ then $\lambda_k (A) \geq \lambda_k (B) $.
\end{corollary}

\subsection{Proof of Theorem~\ref{th:main}}

\begin{proof}
    We start giving an upper bound of the Euclidean norm of $\tilde{a}^\pm_{t+1,i} $~\eqref{eq:defP_a} with the following calculation
    \begin{align}
        \| \tilde{a}^\pm_{t+1,i} \|^2_2 &= 1 \pm \frac{2}{\sqrt{\lambda_{1,t}} }\langle c_t , v_{t,i} \rangle + \frac{1}{\lambda_{t,1}} \\
        & \leq 1 + \frac{2}{\sqrt{\lambda_{t,1}}} + \frac{1}{\lambda_{t,1}}  = \left( 1 + \frac{1}{\sqrt{\lambda_{t,1}}} \right)^2,
    \end{align}
    where we used that $c_t,v_{t,i}\in\mathbb{S}^d$. Thus using the definition of $P_{t,i}$~\eqref{eq:defP_a},
    \begin{align}
        P_{t,i} \geq \left( 1 + \frac{1}{\sqrt{\lambda_{t,1}}}  \right)^{-2} \big( \tilde{a}^+_{t+1,i}(\tilde{a}^+_{t+1,i})^\mathsf{T}  + \tilde{a}^-_{t+1,i} (\tilde{a}^-_{t+1,i})^\mathsf{T} \big) .
    \end{align}
    From the definition of $\tilde{a}^\pm_{t+1,i} $~\eqref{eq:defP_a} we have
    \begin{align}
    \tilde{a}^\pm_{t+1,i} (\tilde{a}^\pm_{t+1,i})^\mathsf{T} = c_t c_t^\mathsf{T} + \frac{1}{\lambda_{t,1}}v_{t,i} v^\mathsf{T}_{t,i} \pm \frac{1}{\sqrt{\lambda_{t,1}}} \left( c_t v_{t,i}^\mathsf{T} + v_{t,i}c_t^\mathsf{T} \right).
    \end{align}
    The above allow us to bound $P_{t,i}$ as 
    \begin{align}
        P_{t,i} &\geq 2 \left( 1 +\frac{1}{\sqrt{\lambda_{t,1}}} \right)^{-2}\left( c_t c^\mathsf{T}_t + \frac{1}{\lambda_{t,1}}v_{t,i} v^\mathsf{T}_{t,i} \right) \\
        &\geq \frac{2}{(1+\sqrt{\lambda_{t,1}})^2} v_{t,i}v^\mathsf{T}_{t,i},
    \end{align}
    where we used $c_t c^\mathsf{T}_t \geq 0$ and $\frac{1}{\lambda_{t,1}}\left( 1 +\frac{1}{\sqrt{\lambda_{t,1}}} \right)^{-2} = (1+\sqrt{\lambda_{t,i}})^{-2} $. Thus, from the abound bound and the definition of $V_t$~\eqref{eq:vt_lemma} we obtain the following matrix inequality
    \begin{align}\label{eq:vt+1_bound}
        V_{t+1} \geq V_t +  \frac{2 w(V_t)}{(1+\sqrt{\lambda_{t,1}})^2} \sum_{i=1}^{d-1} v_{t,i} v^\mathsf{T}_{t,i}.
    \end{align}

    We want to prove the result using induction. The case $t=0$ is immediately satisfied since $\lambda_{0,d} = \lambda_{0,1} \geq \frac{2}{3(d-1)}$.
    Now we will assume that 
    \begin{align}\label{eq:induction}
        \lambda_{t,1} \geq \sqrt{\frac{2}{3(d-1)}\lambda_{t,d}},
    \end{align}
    is satisfied and we want to prove the same inequality for $t+1$. We distinguish cases depending on the growth of the maximum eigenvalue of $V_t$.
\subsection*{Case 1:  $\lambda_{t,d} \geq \lambda_{t,d-1} +  \frac{2 w(V_t)}{(1+\sqrt{\lambda_{t,1}})^2}$}

Using the hypothesis $\lambda_{t,d} \geq \lambda_{t,d-1} +  \frac{2 w(V_t)}{(1+\sqrt{\lambda_{t,1}})^2}$, the fact that $V_t$ diagonalizes in the $\lbrace v_{t,i} \rbrace_{i=1}^d$ basis and the ordering  $\lambda_{t,1}\leq ....\leq \lambda_{t,d-1}\leq \lambda_{t,d}$ we have
\begin{align}
    \lambda_i \left( V_t +  \frac{2 w(V_t)}{(1+\sqrt{\lambda_{t,1}})^2} \sum_{i=1}^{d-1} v_{t,i} v^\mathsf{T}_{t,i} \right) &= \lambda_{t,i} + \frac{2 w(V_t)}{(1+\sqrt{\lambda_{t,1}})^2} \quad \text{for} \quad i= 1,...,d-1 \\
    \lambda_d \left(V_t +  \frac{2 w(V_t)}{(1+\sqrt{\lambda_{t,1}})^2} \sum_{i=1}^{d-1} v_{t,i} v^\mathsf{T}_{t,i} \right) &= \lambda_{t,d} .
\end{align}
Then combining with the mini-max principle for eigenvalues~\ref{cor:minimax_eigenvalues} and using that both sides of~\eqref{eq:vt+1_bound} are real and symmetric, we arrive at
\begin{align}\label{eq:lambdaminbound}
    \lambda_{i,t+1} \geq \lambda_{i,t} + \frac{2 w(V_t)}{(1+\sqrt{\lambda_{t,1}})^2} \quad \text{for} \quad i=1,...,d-1 .
\end{align}
From the expression for $V_{t+1}$~\eqref{eq:vt_lemma}, we deduce that
\begin{align}
    \Tr (V_{t+1}) = \big( \lambda_{t+1,d} + \sum_{i=1}^{d-1} \lambda_{t+1,i} \big) = \lambda_{t,d}+\sum_{i=1}^{d-1} \lambda_{t,i}+2(d-1)w(V_t).
\end{align}
Also from~\eqref{eq:lambdaminbound}
\begin{align}
    \Tr (V_{t+1} ) \geq \lambda_{t+1,d} + \sum_{i=1}^{d-1} \lambda_{t,i} + \frac{2(d-1) w(V_t)}{(1+\sqrt{\lambda_{t,1}})^2} . 
\end{align}
Combining the above we can bound the maximum eigenvalue as
\begin{align}\label{eq:lambda_upper}
    \lambda_{t+1,d} &\leq \lambda_{t,d}+\sum_{i=1}^{d-1} \lambda_{t,i}+2(d-1)w(V_t) - \left( \sum_{i=1}^{d-1} \lambda_{t,i} + \frac{2(d-1) w(V_t)}{(1+\sqrt{\lambda_{t,1}})^2} \right) \\
    &= \lambda_{d,t} + 2(d-1)w(V_t ) \frac{\lambda_{t,1}+2\sqrt{\lambda_{t,1}}}{(1+\sqrt{\lambda_{t,1}})^2}.
\end{align}
Recall that we want to check
\begin{align}
    \lambda_{t+1,1} \geq \sqrt{\frac{2}{3(d-1)}\lambda_{t+1,d}}.
\end{align}
Using~\eqref{eq:lambdaminbound} and~\eqref{eq:lambda_upper} we can square the above and see that it suffices to check
\begin{align}
    \left(\lambda_{1,t} + \frac{2 w(V_t)}{(1+\sqrt{\lambda_{t,1}})^2}\right)^2 \geq \frac{2}{3(d-1)} \lambda_{d,t} + \frac{4w(V_t )(\lambda_{t,1}+2\sqrt{ \lambda_{t,1}})}{3(1+\sqrt{\lambda_{t,1}})^2}.
\end{align}
Multiplying out the terms, this is equivalent to the condition
\begin{align}
    \underbrace{\lambda^2_{1,t} - \frac{2}{3(d-1)} \lambda_{d,t}}_{(i)} + \underbrace{\frac{4 w^2(V_t)}{(1+\sqrt{\lambda_{t,1}})^4}}_{(ii)} + \underbrace{\frac{4w(V_t)}{(1+\sqrt{\lambda_{t,1}})^2}(\lambda_{t,1} - \frac{1}{3}(\lambda_{t,1}+2\sqrt{\lambda_{t,1}})}_{(iii)} \geq 0.
\end{align}
It remains to observe that $(i)$ is positive by induction at time step $t$ (cf.~\eqref{eq:induction}), $(ii)$ is always positive and $(iii)$ is positive for $\lambda_{t,1} \geq 1$ and this is true since $\lambda_{t,1} \geq \lambda_0 \geq 2$, $\lambda_{t,1} $ is non-decreasing in $t$ and $f(x) = x-\sqrt{x}$ is positive for $x\geq 1$.
%

\subsection*{\textbf{Case 2.1}: $ \lambda_{t,1} +  \frac{2 w(V_t)}{(1+\sqrt{\lambda_{t,1}})^2}<\lambda_{t,d} < \lambda_{t,d-1}+ \frac{2 w(V_t)}{(1+\sqrt{\lambda_{t,1}})^2}$}

First we find $k\in\mathbb{N}$, $1<k\leq d-1$ such that
\begin{align}
  \lambda_{t,k-1} +  \frac{2 w(V_t)}{(1+\sqrt{\lambda_{t,1}})^2} \leq  \lambda_{t,d} \leq \lambda_{t,k} + \frac{2 w(V_t)}{(1+\sqrt{\lambda_{t,1}})^2}.
\end{align}
Then using that $\lbrace v_{t,i} \rbrace_{i=1}^d$ are the eigenvectors of $V_t$ we have
\begin{align}
   \lambda_i \left( V_t +  \frac{2 w(V_t)}{(1+\sqrt{\lambda_{t,1}})^2} \sum_{i=1}^{d-1} v_{t,i} v^\mathsf{T}_{t,i} \right) &= \lambda_{t,i} + \frac{2 w(V_t)}{(1+\sqrt{\lambda_{t,1}})^2} \quad \text{for} \quad i= 1,...,k-1 \\
    \lambda_k \left(V_t +  \frac{2 w(V_t)}{(1+\sqrt{\lambda_{t,1}})^2} \sum_{i=1}^{d-1} v_{t,i} v^\mathsf{T}_{t,i} \right) &= \lambda_{t,d} \\
    \lambda_i \left( V_t +  \frac{2 w(V_t)}{(1+\sqrt{\lambda_{t,1}})^2} \sum_{i=1}^{d-1} v_{t,i} v^\mathsf{T}_{t,i} \right) &= \lambda_{t,i-1} + \frac{2 w(V_t)}{(1+\sqrt{\lambda_{t,1}})^2} \quad \text{for} \quad i= k+1,...,d .
\end{align}
Again using the mini-max principle for eigenvalues~\ref{cor:minimax_eigenvalues} and that both sides of~\eqref{eq:vt+1_bound} are real and symmetric
\begin{align}
    \lambda_{t+1,i} &\geq 
        \lambda_{t,i} + \frac{2 w(V_t)}{(1+\sqrt{\lambda_{t,1}})^2} \quad \text{if} \quad i \in \lbrace 1,...,k-1 \rbrace, \\
      \lambda_{t+1,i} &\geq    \lambda_{t,d} \quad \text{if} \quad i= k, \\
      \lambda_{t+1,i} &\geq    \lambda_{t,i-1} + \frac{2 w(V_t)}{(1+\sqrt{\lambda_{t,1}})^2} \quad \text{if} \quad i \in \lbrace k+1,...,d \rbrace.
\end{align}
Thus, using the above inequalities we can bound the trace of $V_{t+1}$ as
\begin{align}
    \Tr (V_{t+1} ) &= \sum_{i=1}^{k-1} \lambda_{t+1,i} + \lambda_{t+1,k} +\sum_{i=k+1}^{d-1} \lambda_{t+1,i} + \lambda_{t+1,d}  \\
    &\geq \left(\sum_{i=1}^{k-1} \lambda_{t,i} \right) + \frac{2(k-1) w(V_t)}{(1+\sqrt{\lambda_{t,1}})^2}  +  \lambda_{t,d} + \left(\sum_{i=k}^{d-2} \lambda_{t,i} \right) + \frac{2(d-k-1) w(V_t)}{(1+\sqrt{\lambda_{t,1}})^2}  + \lambda_{t+1,d}  \\
    & =  \lambda_{t+1,d} + \lambda_{t,d} + \sum_{i=1}^{d-2} \lambda_{t,i} + \frac{2(d-2) w(V_t)}{(1+\sqrt{\lambda_{t,1}})^2} \\
    & \geq \lambda_{t+1,d}  + \sum_{i=1}^{d-1} \lambda_{t,i} + \frac{2(d-2) w(V_t)}{(1+\sqrt{\lambda_{t,1}})^2},
\end{align}
where in the last bound we used simply $\lambda_{t,d} \geq \lambda_{t,d-1}$.
From the expression of $V_{t+1}$~\eqref{eq:vt_lemma}
\begin{align}
    \Tr (V_{t+1})  = \lambda_{t,d}+\sum_{i=1}^{d-1} \lambda_{t,i}+2(d-1)w(V_t),
\end{align}
and combining with the previous bound we obtain
\begin{align}
    \lambda_{t+1,d} &\leq \lambda_{d,t} + 2(d-1)w(V_t ) -\frac{2(d-2)w(V_t)}{(1+\sqrt{\lambda_{t,1}})^2} \\
    & = \lambda_{d,t} + \frac{2w(V_t )}{(1+\sqrt{\lambda_{t,1}})^2}\left( \lambda_{t,1}(d-1) + 2\sqrt{\lambda_{t,1}}(d-1) + 1 \right) \\
    & \leq \lambda_{d,t} + \frac{2(d-1)w(V_t )}{(1+\sqrt{\lambda_{t,1}})^2}\left( \lambda_{t,1} + 2\sqrt{\lambda_{t,1}} + 1 \right).
\end{align}
Again to check
\begin{align}
    \lambda_{t+1,1} \geq \sqrt{\frac{2}{3(d-1)}\lambda_{t+1,d}},
\end{align}
 we can square both sides and using the above bounds on $\lambda_{t+1,1}$ and $\lambda_{t+1,d}$ it suffices to check
\begin{align}
    \left(\lambda_{1,t} + \frac{2 w(V_t)}{(1+\sqrt{\lambda_{t,1}})^2}\right)^2 \geq \frac{2}{3(d-1)} \lambda_{d,t} + \frac{4w(V_t )(\lambda_{t,1}+2\sqrt{ \lambda_{t,1}}+1)}{3(1+\sqrt{\lambda_{t,1}})^2}.
\end{align}
Multiplying out the terms, this is equivalent to the condition
\begin{align}
    \underbrace{\lambda^2_{1,t} - \frac{2}{3(d-1)} \lambda_{d,t}}_{(i)} + \underbrace{\frac{4 w^2(V_t)}{(1+\sqrt{\lambda_{t,1}})^4}}_{(ii)} + \underbrace{\frac{4w(V_t)}{(1+\sqrt{\lambda_{t,1}})^2}(\lambda_{t,1} - \frac{1}{3}(\lambda_{t,1}+2\sqrt{\lambda_{t,1}}+1)}_{(iii)} \geq 0.
\end{align}
Finally, we observe that $(i)$ is positive by induction at time step $t$, $(ii)$ is always positive and $(iii)$ is positive since $\lambda_{t,1} \geq 2$, $\lambda_{t,1} \geq 2$ is non-decreasing and $f(x) = 2x-2\sqrt{x}-1$ is positive for $x \geq 2$.


\subsection*{\textbf{Case 2.2}:  $\lambda_{t,d} \leq \lambda_{t,1} +  \frac{2 w(V_t)}{(1+\sqrt{\lambda_{t,1}})^2}$}

From the statement of the theorem we have 
\begin{align}
    \lambda_{t,1} \geq \lambda_0 \geq \sqrt{\frac{2}{3(d-1)}}2dC+\frac{2}{3(d-1)}.
\end{align}
Multiplying both sides by $\lambda_{t,1}$ we have
\begin{align}
     \lambda^2_{t,1} &\geq  \sqrt{\frac{2}{3(d-1)}}2dC\lambda_{t,1}+\frac{2}{3(d-1)}\lambda_{t,1} \\
     & \geq \frac{2}{3(d-1)}\left(2dC\sqrt{\lambda_{t,d}} + \lambda_{t,1} \right) \\
     & \geq \frac{2}{3(d-1)}\left(2dw(V_t) + \lambda_{t,1} \right) \\
     & = \frac{2}{3(d-1)}\left(2(d-1)w(V_t) + 2 w(V_t )+ \lambda_{t,1} \right) \\
     & \geq \frac{2}{3(d-1)}\left(2(d-1)w(V_t) + \lambda_{t,d} \right) ,
\end{align}
where the second inequality follows from the induction hypothesis $\lambda_{t,1}\geq \sqrt{\frac{2}{3(d-1)}\lambda_{t,d}}$, the third inequality from $w(V_t ) \leq C\sqrt{\lambda_{t,d}}$ and the fourth inequality from the assumption $\lambda_{t,d} \leq \lambda_{t,1} +  \frac{2 w(V_t)}{(1+\sqrt{\lambda_{t,1}})^2} \leq \lambda_{t,1}+2w(V_t)$. Thus taking the square root in both sides we have
\begin{align}
    \lambda_{t,1} \geq \sqrt{\frac{2}{3(d-1)}\left( \lambda_{t,d} + 2(d-1) w(V_t ) \right)},
\end{align}
and the induction at time step $t+1$,  $\lambda_{t+1,1}\geq \sqrt{\frac{2}{3(d-1)}\lambda_{t+1,d}}$ follows from the bounds
\begin{align}
    \lambda_{t+1,1} &\geq \lambda_{t,1}, \\
    \lambda_{t+1,d} & \leq \lambda_{t,d} + 2 w(V_t ) (d-1) ,
\end{align}
where we used the inequalities~\eqref{eq:pertboundsAB} and the definition of $V_t$~\eqref{eq:vt_lemma}.
\end{proof}

\subsection{Alternative proof for special case $d=2$}\label{ap:sec3_case2}
We present an alternative proof of the previous result for the particular case of $d=2$ where the main difference is that we provide an exact calculation of the eigenvalues of the matrix $V_{t+1}$. This technique can slightly improve the constant on the lower bound $\lambda_{\min}(V_t) = \Omega (\sqrt{\lambda_{\max}(V_t))}$.
\begin{theorem}
    Let $\lbrace c_t \rbrace_{t=0}^\infty \subset \mathbb{S}^2$ be a sequence of normalized vectors and $\omega: \text{P}^2_+ \rightarrow \mathbb{R}_{\geq 0}$ a function such that 
    \begin{align}
        \omega (X) \leq C\sqrt{\lambda_{\max} (X)},
    \end{align}
     for a constant $C > 0$ and any $X\in \text{P}^2_+$. Let $\lambda_0 \geq 4\sqrt{2}C+2$, and define a sequence of matrices
     $\lbrace V_t \rbrace_{t=0}^\infty \subset \mathbb{R}^{2\times 2}$ as
       \begin{align}\label{eq:vt_d2}
         V_0 := \lambda_0\mathbb{I}_{2\times2}, \quad      V_{t+1} := V_t + \omega ( V_t ) \left(a^+_{t+1}  (a^+_{t+1})^\mathsf{T}  +  a^-_{t+1}  (a^-_{t+1})^\mathsf{T} \right), 
       \end{align}
       where $a^+_{t+1},a^-_{t+1}\in\mathbb{S}^2$ are defined as
     \begin{align}
        a^+_{t+1} := \frac{c_t + \frac{1}{\sqrt{\lambda_{\min}(V_t)}}v_{t,\min}}{\sqrt{1+\frac{2 \langle c_t, v_{t,\min} \rangle }{\sqrt{\lambda_{\min}(V_t)}} + \frac{1}{\lambda_{\min}(V_t)}}}, \quad  a^-_{t+1} := \frac{c_t - \frac{1}{\sqrt{\lambda_{\min}(V_t)}}v_{t,\min}}{\sqrt{1-\frac{2 \langle c_t, v_{t,\min} \rangle }{\sqrt{\lambda_{\min}(V_t)}} + \frac{1}{\lambda_{\min}(V_t)}}}, 
        \end{align}
    and $v_{\min,t}$ is the normalized eigenvector corresponding to the minimum eigenvalue . Then we have
    \begin{align}
        \lambda_{\min}(V_t) \geq \sqrt{2\lambda_{\max}(V_t)} \quad \text{for all}\quad t\geq 0.
    \end{align}
  
\end{theorem}

\begin{proof}
 To simplify the notation in the proof we define
 \begin{align}
     \lambda_{\min,t} := \lambda_{\min}(V_t) , \quad \lambda_{\max,t} : = \lambda_{\max}(V_t),
 \end{align}
 with corresponding normalized eigenvectors $v_{t,\min},v_{t,\max}\in\mathbb{S}^2$.

 We are going to prove the result by induction. At time step $t=0$ we have $\lambda_{\min,t} = \lambda_{\max,t} = \lambda_0$ and the inequality $\lambda_0 \geq \sqrt{2\lambda_0}$ holds since $\lambda_0 \geq 2$.
 Then at time step $t\geq 1$ we can express $V_t$ as
\begin{align}
    V_t = \begin{pmatrix}
        \lambda_{t,\min} & 0 \\
        0 & \lambda_{t,\max}
    \end{pmatrix},
\end{align}
using the basis $\lbrace v_{t,\min}, v_{t,\max}\rbrace$. Then we assume that the inequality $\lambda_{t,\min} \geq \sqrt{2\lambda_{\max,t}}$ is satisfied and we want to check the same inequality at time step $t+1$. We express the vectors that update $V_{t+1}$~\eqref{eq:vt_d2} on the same basis as
\begin{align}\label{eq:atminmaxbasis}
    a^+_{t+1} = \begin{pmatrix}
        \langle a^+_{t+1}, v_{t,\min} \rangle  \\
         \langle a^+_{t+1} ,  v_{t,\max} \rangle
    \end{pmatrix} \quad  a^-_{t+1} = \begin{pmatrix}
          \langle a^-_{t+1} ,  v_{t,\min} \rangle  \\
          \langle a^-_{t+1} ,  v_{t,\max} \rangle
    \end{pmatrix},
\end{align}
and define the following quantities,
    \begin{align}
        x_t &:= \langle a^+_{t+1}, v_{t,\min} \rangle^2 + \langle a^-_{t+1} , v_{t,\min} \rangle^2, \\
        y_t &:= \langle a^+_{t+1} , v_{t,\max} \rangle^2 + \langle a^-_{t+1} ,  v_{t,\max} \rangle^2 , \\
        z_t &:= \langle a^+_{t+1} , v_{t,\min} \rangle  \langle a^+_{t+1} , v_{t,\max} \rangle + \langle a^-_{t+1} , v_{t,\min} \rangle  \langle a^-_{t+1} , v_{t,\max} \rangle  . 
    \end{align}
Using that $a^+_{t+1},a^-_{t+1}\in\mathbb{S}^2$  we get the following relation
\begin{align}\label{eq_relationxy}
    y_t = 2 - x_t .
\end{align}
Thus the perturbation at time step $t+1$ can be written as
\begin{align}
a^+_{t+1}  (a^+_{t+1})^\mathsf{T}  +  a^-_{t+1}  (a^-_{t+1})^\mathsf{T}  = \begin{pmatrix}
         x_t & z_t \\
         z_t & y_t
     \end{pmatrix},
\end{align}     
and $V_{t+1}$ as
\begin{align}
    V_{t+1} = \begin{pmatrix}
        \lambda_{t,\min} + \omega ( V_t ) x_t & \omega ( V_t )z_t \\
        \omega ( V_t )z_t & \lambda_{t,\max} + \omega ( V_t )(2-x_t)
    \end{pmatrix}.
\end{align}
In order to analyze the eigenvalue of $V_{t+1}$ we have to control the overlap $\langle v_{t,\min} , c_t \rangle$, so we define the following variable
    \begin{align}
        \alpha_t : = \langle v_{t,\min} , c_t \rangle \in [-1,1],
    \end{align}
and using that
\begin{align}
    \langle v_{t,\min} , a^+_{t+1} \rangle = \frac{\alpha_t + \frac{1}{\sqrt{\lmin}}}{\sqrt{1 +\frac{2}{\sqrt{\lmin}}\alpha_t + \frac{1}{\lmin}}}, \quad \langle v_{t,\min} , a^-_{t+1} \rangle = \frac{\alpha_t - \frac{1}{\sqrt{\lmin}}}{\sqrt{1-\frac{2}{\sqrt{\lmin}}\alpha_t + \frac{1}{\lmin}}}
\end{align}
we can express $x_t,z_t$ in terms of $\alpha_t$ as
\begin{align}
    x_t(\alpha_t) &= \frac{\big( \alpha_t + \frac{1}{\sqrt{\lmin}}\big)^2}{1+\frac{2}{\sqrt{\lmin}}\alpha_t + \frac{1}{\lmin}} + \frac{\big( \alpha_t - \frac{1}{\sqrt{\lmin}}\big)^2}{1-\frac{2}{\sqrt{\lmin}}\alpha_t + \frac{1}{\lmin}}, \\
    z_t(\alpha_t) &= \left( \frac{\alpha_t + \frac{1}{\sqrt{\lmin}}}{1+\frac{2}{\sqrt{\lmin}}\alpha_t + \frac{1}{\lmin}} + \frac{ \alpha_t - \frac{1}{\sqrt{\lmin}}}{1-\frac{2}{\sqrt{\lmin}}\alpha_t + \frac{1}{\lmin}}\right)\sqrt{1-\alpha_t^2}.
\end{align}
Then we can directly compute the minimum eigenvalue of $V_{t+1}$ as
\begin{align}\label{eq:exact_minimumeig}
    \lambda_{\min,t+1}(\alpha_t) = \frac{\lambda_{\max,t}+\lmin}{2} + \omega ( V_t ) - \frac{1}{2}\sqrt{(\lambda_{\max,t} - \lambda_{\min,t}+2\omega ( V_t )(1-x_t(\alpha_t)) )^2 + 4\omega^2 ( V_t ) z_t^2(\alpha_t)}.
\end{align}
Now we define the difference of eigenvalues at time step $t$,
\begin{align}
    \phi_t := \lambda_{\max,t} - \lambda_{\min,t},
\end{align}
and analyze two different regimes for the induction to hold at time step $t+1$.


\subsection*{\textbf{Case 1}: $\phi_t \geq 2\omega ( V_t )$}

For this case we want to justify that $\lambda_{\min,t+1}(\alpha_t) \geq \lambda_{\min,t+1}(0)$ for $\alpha_t\in[-1,1]$, and later prove the induction using this lower bound. Defining the following
\begin{align}
    f(\alpha_t) &:= (\phi_t +2\omega ( V_t )(1-x(\alpha_t)) )^2 + 4\omega^2 ( V_t )z^2(\alpha_t) \\
    & = \phi^2_t + 4\omega ( V_t )\big( \phi_t ( 1 - x(\alpha_t)) + \omega ( V_t )(1-x(\alpha_t))^2 + \omega ( V_t )z^2 ( \alpha_t ) \big),
\end{align}
and comparing with the exact minimum eigenvalue~\eqref{eq:exact_minimumeig} we see that it suffices to prove that $f(\alpha_t)$ achieves a maximum at $\alpha_t=0$ in the range $\alpha_t\in [-1,1]$ in order to have $\lambda_{\min,t+1}(\alpha_t) \geq \lambda_{\min,t+1}(0)$ for all $\alpha_t\in[-1,1]$. A direct computation shows that
\begin{align}
    f(\alpha_t ) = \phi_t^2 + 4\omega ( V_t )g(\alpha_t), \quad g(\alpha_t) :=\frac{2\big(\lmin-1\big)\lmin\phi_t \alpha_t^2 +(1-\lmin^2)\phi_t - (1-\lmin )^2\omega ( V_t )}{4\lmin \alpha_t^2 - \big( 1 + \lmin \big)^2 }.
\end{align}
We have that $g(\alpha_t)$ is of the form $g(x) = \frac{ax^2+b}{cx^2+d}$ which has an unique maximum at $x=0$ if $ad-bc < 0$. Then identifying the coefficients we have to check that
\begin{align}
  p:=  2(1-\lmin)\lmin (1+\lmin)^2 \phi_t - 4\lmin ( 1-\lmin^2)\phi_t + 4 \lmin (1-\lmin )^2 \omega ( V_t )< 0 .
\end{align}
Using that $\lambda_{\min,t}\geq 2$ and $\omega ( V_t ) \leq \frac{\phi_t}{2}$ we can bound the above as
\begin{align}
    p \leq \phi_t \left( 2(1-\lmin)\lmin (1+\lmin)^2  - 4\lmin ( 1-\lmin^2) + 2 \lmin (1-\lmin )^2 \right).
\end{align}
Finally summing all the terms we get,
\begin{align}
    p \leq -2\phi_t \left(\lambda_{\min,t}^2(1-\lambda_{\min,t})^2 \right) \leq 0 ,
\end{align}
where the inequality follows from $\phi_t \geq 0$. Thus, we conclude that 
\begin{align}\label{eq:minimumeig_inequality}
\lambda_{\min,t+1}(\alpha_t) \geq \lambda_{\min,t+1}(0).
\end{align}
Using that 
\begin{align}
x_t (0 ) = \frac{2}{1+\lambda_{\min,t}} , \quad z_t(0) = 0 ,
\end{align}
and $\phi_t \geq 2\omega (V_t ) $ we have
\begin{align}\label{eq:lambdamin0}
    \lambda_{\min,t+1}( 0 ) = \lambda_{\min,t} + \frac{2\omega ( V_t )}{1+\lambda_{\min,t}}.
\end{align}
Thus, using~\eqref{eq:minimumeig_inequality}\eqref{eq:lambdamin0} and 
\begin{align}
\Tr ( V_{t+1} ) = \lambda_{\min,t+1}(\alpha_t) + \lambda_{\max,t+1}(\alpha_t) = \lambda_{\max,t} + \lambda_{\min,t} + 2\omega ( V_t ),
\end{align}
we can upper bound the maximum eigenvalue as
\begin{align}\label{eq:maximumeig_bound}
    \lambda_{\max,t+1}(\alpha_t) &= \lambda_{\max,t} + \lambda_{\min ,t} - \lambda_{\min,t+1}(\alpha_t) + 2\omega ( V_t ) \\
    &\leq \lambda_{\max,t} + \lambda_{\min ,t} - \lambda_{\min,t+1}(0) + 2\omega ( V_t ) \\
     &= \lambda_{\max,t}+ 2\omega ( V_t )\frac{\lambda_{\min,t}}{1+\lambda_{\min,t}}.
\end{align}
Finally in order to check the induction step at $t+1$, $\lambda_{\min,t+1} \geq  \sqrt{2\lambda_{\max,t+1}}$ we can use the bounds~\eqref{eq:minimumeig_inequality}\eqref{eq:maximumeig_bound} and it suffices to check
\begin{align}
    \lambda_{\min,t} + \frac{2\omega ( V_t )}{1+\lambda_{\min,t}} \geq  \sqrt{2\left(\lambda_{\max,t} + 2\omega ( V_t )\frac{\lambda_{\min,t}}{1+\lambda_{\min,t}}\right)}.
\end{align}
Squaring both sides and rearranging we obtain
\begin{align}
    \underbrace{\lambda^2_{\min,t} - 2\lambda_{\max,t}}_{(i)} + \underbrace{\frac{4\omega^2 ( V_t )}{(1+\lambda_{\min,t})^2}}_{(ii)} \geq 0 ,
\end{align}
where $(i)$ is positive because we assume $\lambda_{\min,t} \geq \sqrt{2\lambda_{\max,t}}$ to hold at time step $t$ and  $(ii)$ is always positive. This concludes the case $\phi_t \geq 2 \omega ( V_t )$ induction.


\subsection*{\textbf{Case 2}: $\phi_t \leq 2\omega (V_t ) $}

From the definition of $\lambda_0$ we have
\begin{align}
    \lambda_{\min,t} &\geq \lambda_{0} \geq 4\sqrt{2}C+2.
    \end{align}
Multiplying both sides by $\lambda_{\min,t}$,    
    \begin{align}
    \lambda_{\min,t}^2 &\geq 4\sqrt{2}C\lambda_{\min,t}+2\lambda_{\min,t} \\
    &\geq 2\left( 4C\sqrt{\lambda_{\max,t}} + \lambda_{\min,t} \right) \hspace{20mm} \\
    &\geq 2\left( \lambda_{\min,t} + 4\omega (V_t ) \right) \\
    &\geq 2 \left( \lambda_{\min,t} +\phi_t+ 2\omega (V_t ) \right) \\
    & = 2\left( \lambda_{\max,t} + 2\omega (V_t ) \right) ,
\end{align}
where in the second inequality we used the induction hypothesis $\lambda_{\min,t}\geq \sqrt{2\lambda_{\max,t}}$, the third inequality $\omega(V_t) \leq C\sqrt{\lambda_{\max,t}}$, the fourth inequality $\phi_t\leq 2\omega(V_t )$ and the last equality the definition $\phi_t = \lambda_{\max,t} - \lambda_{\min,t}$.
Thus, we have
\begin{align}
    \lambda_{\min,t} \geq \sqrt{2(\lambda_{\max,t}+2\omega (V_t ))},
\end{align}
and the inequality at time step $t+1$, $\lambda_{\min,t} \geq \sqrt{2\lambda_{\max,t+1}}$ follows from the bounds
\begin{align}
    \lambda_{\min,t+1} &\geq \lambda_{\min,t}, \\
    \lambda_{\max,t+1} &\leq \lambda_{\max,t} + 2\omega (V_t ),
\end{align}
where we used~\eqref{eq:pertboundsAB} with $\lambda_{\min}(a_ta_t^\mathsf{T}) = \lambda_{\min}(b_tb_t^\mathsf{T}) = 0 $ and $\lambda_{\max}(a_ta_t^\mathsf{T}) = \lambda_{\max}(b_tb_t^\mathsf{T}) = 1$.
\end{proof}

\section{Proofs of Section 6}\label{ap:sec6}
In our regret analysis, we need to control the distance between the regularized least squares estimator and the normalized version. The Lemma below is the technical result that we use.
\begin{lemma}\label{lem:dist_action_centre}
    Given two normalized vectors $c,v\in\mathbb{S}^d$, a positive constant $\lambda > 1$ and the following vectors
    \begin{align}
        \widetilde{a}^{\pm} = c \pm \frac{1}{\sqrt{\lambda}}v, \quad {a}^\pm = \frac{\widetilde{a}^\pm}{\| \widetilde{a}^\pm \|_2}.
    \end{align}
    Then we have 
    \begin{align}
        \|{a}^\pm - c\|_2^2 \leq \frac{2}{\lambda}.
    \end{align}
\end{lemma}

\begin{proof}
We are going to give the proof for ${a}^-$. The one for ${a}^+$ follows from an identical calculation.

First, we can relate the distance to the inner product of the vectors as
    \begin{align}\label{eq:acdistance}
        \|{a}^- - c\|_2^2 = \langle {a}^- - c , {a}^- - c \rangle = 2 - 2\langle {a}^-, c \rangle.
    \end{align}
    Using the normalization factor is
    \begin{align}
        \| \widetilde{a}^- \|_2^2 =  1 - \frac{2}{\sqrt{\lambda}}\langle c, v \rangle + \frac{1}{\lambda},
    \end{align}
    then
    \begin{align}
      \langle {a}^-, c \rangle =   \frac{1-\frac{1}{\sqrt{\lambda}}\alpha}{\sqrt{1+\frac{1}{\lambda}-\frac{2}{\sqrt{\lambda}}\alpha}}, \quad \text{where} \quad \alpha: = \langle c ,v \rangle \in [-1,1].
    \end{align}
    In order to study the behavior of~\eqref{eq:acdistance} in terms of the overlap $\langle c , v \rangle$ we define
    \begin{align}\label{eq:def_distance_lambda}
        f(\alpha,\lambda) := 2 \left( 1 - \frac{1-\frac{1}{\sqrt{\lambda}}\alpha}{\sqrt{1+\frac{1}{\lambda}-\frac{2}{\sqrt{\lambda}}\alpha}} \right) 
    \end{align}
    and we are going to check the maximum in the range $\alpha\in [-1,1]$. Note that $f(\alpha,\lambda) =  \|{a}^- - c\|_2^2$.
    Taking the derivative respect to $\alpha$,
    \begin{align}
        \frac{\partial f(\alpha,\lambda)}{\partial \alpha} = 2\frac{\frac{1}{\sqrt{\lambda}}\sqrt{1+\frac{1}{\lambda}-\frac{2}{\sqrt{\lambda}}\alpha} - \left( 1 - \frac{1}{\sqrt\lambda}\alpha \right)\frac{1}{\sqrt{\lambda}\sqrt{1+\frac{1}{\lambda}-\frac{2}{\sqrt{\lambda}}\alpha}}}{1+\frac{1}{\lambda}-\frac{2}{\sqrt{\lambda}}\alpha}
    \end{align}
    Then we can find the extremal points of $f(\alpha)$ as
    \begin{align}
        \frac{\partial f(\alpha,\lambda)}{\partial \alpha} &= 0 \Leftrightarrow 1+\frac{1}{\lambda}-\frac{2}{\sqrt{\lambda}}\alpha = 1 - \frac{1}{\sqrt\lambda}\alpha  \\
        &\Leftrightarrow \alpha = \frac{1}{\sqrt{\lambda}}
    \end{align}
    Using that $\lambda > 1$ we have the following inequalites
    \begin{align}
        \frac{\partial f(\alpha,\lambda)}{\partial \alpha} \bigg\vert_{\alpha=1} = -\frac{2}{\lambda\left(1-\frac{1}{\sqrt\lambda}\right)^2} < 0, \quad \frac{\partial f(\alpha,\lambda)}{\partial \alpha}\bigg\vert_{\alpha=0} = \frac{2}{\sqrt{\lambda}\sqrt{1+\frac{1}{\lambda}}}\left( 1 - \frac{1}{1+\frac{1}{\lambda}} \right) > 0.
    \end{align}
    Thus, $\alpha = \frac{1}{\sqrt{\lambda}}$~\eqref{eq:def_distance_lambda} is a maximum and it is achieved at
    \begin{align}
        f\left( \frac{1}{\sqrt{\lambda}} ,\lambda \right) = 2\left( 1- \sqrt{1-\frac{1}{\lambda}} \right).
    \end{align}
    Using the original definition of $f(\alpha,\lambda)$ we can conclude that 
    \begin{align}
         \|{a}^- - c\|_2^2 \leq 2\left( 1- \sqrt{1-\frac{1}{\lambda}} \right) \leq \frac{2}{\lambda} ,
    \end{align}
    where the last inequality follows from $\lambda > 1$, and $1-x \leq \sqrt{1-x}$ for $0\leq x < 1$.
\end{proof}

In our regret analysis, we will check the scaling of $\lambda_{\max} ( V_t )$ and to do it we will prove a differential inequality involving  $\lambda_{\max} ( V_t )$. The Lemma below will help us to solve that inequality.

\begin{theorem}[\cite{MichelPetrovitch1901}]\label{th:dif_equation_ineq}
If $u$ satisfies the differential inequality $\frac{\text{d}u(t)}{\text{d}t} \lesseqgtr f(u(t),t)$, and $y$ is the solution to the ordinary differential equation (ODE) $\frac{\text{d}y(t)}{\text{d}t}=f(y(t),t)$ under the boundary condition $u(t_0)=y(t_0)$, then:
\begin{align}
    \forall t > t_0 , u(t) \lesseqgtr y(t) .
\end{align}
\end{theorem}

\subsection{Proof of Theorem~\ref{th:regret_bound_d2}}

\begin{proof}
To simplify notation we use
\begin{align}
    \lambda_{\min,t} := \lambda_{\min}( V_t ), \quad  \lambda_{\max,t} := \lambda_{\max}( V_t )
\end{align}
First, we fix
\begin{align}\label{eq:lambda_def}
    \lambda_0 = \max\left\lbrace2 , \sqrt{\frac{2}{3(d-1)}}\frac{d}{6\sqrt{d-1}} + \frac{2}{3(d-1)} \right\rbrace,
\end{align}
and later we will justify this choice.
Note that the regret can be written as
\begin{align}
    \text{Regret}(T) =   \frac{1}{2}\sum_{t= 1}^{\widetilde{T}} \sum_{i = 1}^{d-1} \left( \| \theta - a^+_{t,i} \|_2^2 + \| \theta - a^{-}_{t,i} \|_2^2  \right) ,
\end{align}
so to upper bound the regret, we need to quantify the distance $\| \theta - a^+_{t,i} \|_2^2$ between the unknown parameter and the actions we select at each batch $t$. The history up to batch $t\leq \widetilde{T}$ is defined as
\begin{align}
    \mathcal{H}_{t} :=  \big( r^+_{s,1}, a^+_{s,1},r^-_{s,1}, a^-_{s,1},...,r^+_{s,d-1}, a^+_{s,d-1},r^-_{s,d-1}, a^-_{s,d-1}\big)_{s=1}^{t} .
\end{align}
Lemma~\ref{lem:confidence_region_weighted} gives us this distance with a certain probability under the assumption that the event
    \begin{align}\label{eq:event_Gt}
     G_t := & \lbrace \big( \mathcal{H}_{t-1}, a^\pm_{s,i} \big)  :  
     \sigma^2_s(a^\pm_{s,i} \big)) \leq \hat{\sigma}^2_s(\mathcal{H}_{s-1},a^\pm_{s,i} )\ \forall {s \in [t]} \rbrace,
\end{align}
holds. 
Using the definition of the subgaussian parameter for the noise $\epsilon_t$~\eqref{eq:vanishing_noise} we can upper bound it for our choice of actions as
\begin{align}\label{eq:noise_bound}
    \sigma^2_{t}(a^\pm_{t,i}) &\leq 1-\langle a^\pm_{t,i}, \theta  \rangle^2 = 1-(1-\frac{1}{2} \| \theta - a^\pm_{t,i} \|_2^2)^2 \nonumber \\
    &=  \| \theta - a^\pm_{t,i} \|_2^2 - \frac{1}{4} \| \theta - a^\pm_{t,i} \|_2^4 \leq \| \theta - a^\pm_{t,i} \|_2^2. 
\end{align}
Thus, we need to use an estimator of the form~\eqref{eq:variance_estimator} that upper bounds the distance $\| \theta - a^\pm_{t,i} \|_2^2$. This quantity depends on the unknown parameter $\theta$, thus we can not guarantee that $G_t$ holds with probability one at each time step $t$. The other event that we will need to hold to quantify the distance is the following 
\begin{align}\label{eq:event_Et}
    E_t := \lbrace \mathcal{H}_{t} : \forall s \in \left[ t\right], \theta \in \mathcal{C}_{s} \rbrace ,
\end{align}
which is guaranteed to hold with probability at least $1-\delta$ if $G_t$ holds by Lemma~\ref{lem:confidence_region_weighted}. We leave the probability computation that both events hold for the end of the proof.

\begin{figure}\label{fig:regret_distances}
\centering
\begin{overpic}[percent,width=0.6\textwidth]{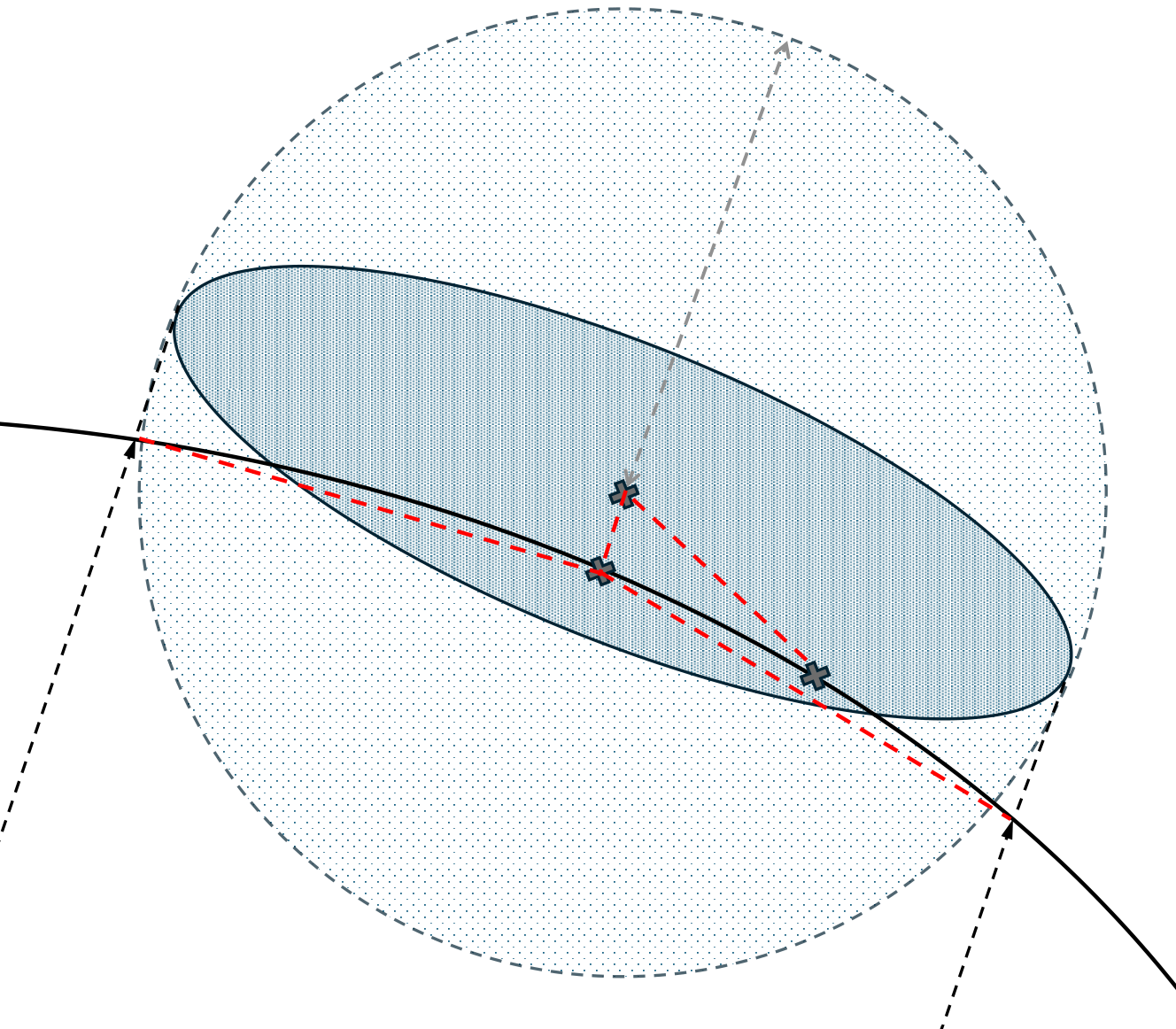}
\put(7,34){ $a^+_t$}
\put(78,12){ $a^-_t$}
\put(67,25){ $\theta$}
\put(54,46){ $\widetilde{\theta}^{\text{w}}_{t-1}$}
\put(48,34){ ${\theta}^{\text{w}}_{t-1}$}
\put(52,63){\rotatebox{70}{\small$r= \sqrt{\frac{\beta_{t-1,\delta}}{\lambda_{\min,t-1}}}$}}
\put(67,65){ $\mathbb{B}_r^2 (\widetilde{\theta}^{\text{w}}_{t-1}) $}
\put(35,55){ $\mathcal{C}_{t-1}$}
\end{overpic}
\caption{Sketch for the triangle inequality used to bound $ \| \theta - a^\pm_{t,i} \|_2 $ in~\eqref{eq:theta_at_bound}. The red lines represent the distances $(i),(ii)$ and $(iii)$. Under the event $\theta\in\mathcal{C}_{t-1}$ we can use $\mathcal{C}_{t-1} \subseteq \mathbb{B}^d_r (\tilde{\theta}^w_{t-1} ) $ with $r$ being the longest axis of the ellipsoid and bound all distances by the diameter $2r$.}
\label{fig:scheme_regret}
\end{figure}
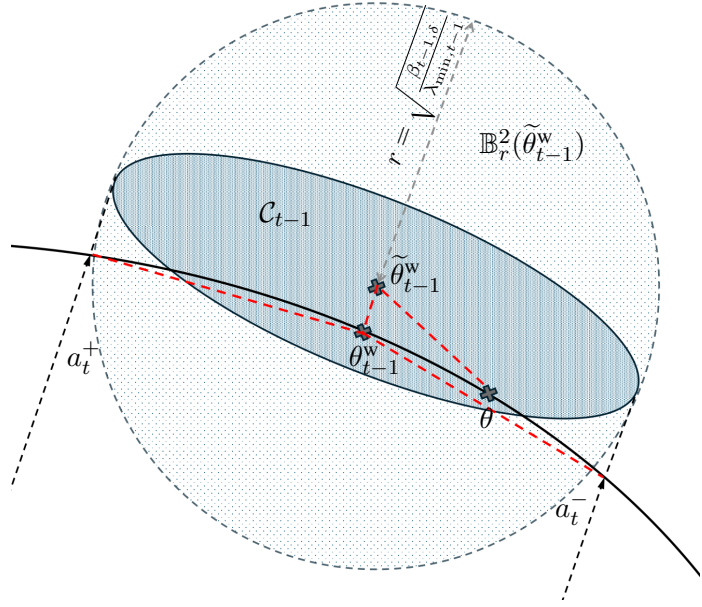

\smallskip


\subsection*{Choosing $w(V_T)$ and instantaneous regret bound under $E_t$}

For now, we will assume that $E_t$ always holds to justify our choice of weights through a bound on $\| \theta - a^\pm_{t,i} \|_2$. This also will allow us to bound the instantaneous regret.  From triangle inequality we have
\begin{align}\label{eq:theta_at_bound}
    \| \theta - a^\pm_{t,i} \|_2 &\leq \| \theta - \tilde{\theta}^{\text{w}}_{t-1} \|_2 + \| \tilde{\theta}^{\text{w}}_{t-1}  - \theta^{\text{w}}_{t-1} \|_2 + \| {\theta}^{\text{w}}_{t-1} -a^\pm_{t,i} \|_2 \nonumber \\
    &\leq \underbrace{\sqrt{\frac{\beta_{t-1,\delta}}{\lambda_{\min,{t-1}}}}}_{(i)} + \underbrace{\sqrt{\frac{\beta_{t-1,\delta}}{\lambda_{\min,{t-1}}}}}_{(ii)} + \underbrace{\sqrt{\frac{2}{\lambda_{\min,{t-1}}}}}_{(iii)} \nonumber
    \\ &\leq 3\sqrt{\frac{\beta_{t-1,\delta}}{\lambda_{\min,{t-1}}}},
\end{align}
where for the above bounds we use:
\begin{itemize}
    \item $(i).$ Since $\theta\in\mathcal{C}_{t-1}$, by definition of $\mathcal{C}_{t-1}$~\eqref{eq:confidence_region} we have $\| \theta -  \tilde{\theta}^{\text{w}}_{t-1} \|_{V_{t-1}} \leq \sqrt{\beta_{t-1,\delta}}$ and then the inequality follows from $\lambda_{\min,t-1}\mathbb{I}_{d\times d}\leq V_{t-1}$.
    \item  $(ii).$ Since $\theta\in\mathcal{C}_{t-1}$ and $\theta\in\mathbb{S}^d$ this implies that $\mathbb{B}^d_r (\tilde{\theta}^w_{t-1} )\cap \mathbb{S}^d$ is non-empty with $r = \sqrt{\frac{\beta_{t-1,\delta}}{\lambda_{\min,t-1}}}$ (the longest axis of the ellipsoid). Then using that $\mathcal{C}_{t-1}\subseteq \mathbb{B}^d_r (\tilde{\theta}^w_{t-1} )$ and ${\theta}^w_{t-1} $ is the normalized vector of $\tilde{\theta}^w_{t-1}$ which is the center of $\mathcal{C}_{t-1}$  we get ${\theta}^w_{t-1} \in\mathbb{B}^d_r (\tilde{\theta}^w_{t-1} )$.
    \item For $(iii).$ We apply Lemma~\ref{lem:dist_action_centre} using the expression of $a^\pm_t$~\eqref{eq:general_update} with $c= {\theta}^{\text{w}}_{t-1}$ and $\lambda= \lambda_{\min,t-1}$.
\end{itemize}
The last inequality follows from $\beta_{t} \geq 2$. For $t=0$,
\begin{align}
    \beta_{0,\delta} = \left(\sqrt{\lambda_0} + \sqrt{2\log\left( \frac{1}{\delta}\right)} \right).
\end{align}
Note that at batch $t$ we use that $E_{t-1}$ holds instead of $E_t$ since we want to define $\hat{\sigma}_t$ such that only depends on past information up to time step $t-1$. Thus, a choice that will guarantee the event $G_t$ to hold under the assumption that $E_{t-1}$ holds is $\hat{\sigma}^2_t  := \frac{9\beta_{t-1,\delta}}{\lambda_{\min,t-1}}$. However our particular update of $V_t$ with the choices of actions $a^+_{t,i},a^-_{t,i}$ guarantees that 
\begin{align}\label{eq:minmaxsqrtrelation}
    \lambda_{\min,t} \geq \sqrt{\frac{2}{3(d-1)}\lambda_{\max,t}},
\end{align}
if we use Theorem~\ref{th:main} (later we will check that we are under the right assumptions to use it). It is important to note that the above bound is independent of the events $G_t,E_t$ and it is a consequence only of our particular choice of actions~\eqref{eq:general_update}. Combining~\eqref{eq:minmaxsqrtrelation} and~\eqref{eq:theta_at_bound} 
\begin{align}
     \| \theta - a^\pm_{t,i} \|_2^2 \leq 9\frac{\beta_{t-1,\delta}}{\lambda_{\min , t-1}} \leq  \frac{9\sqrt{3}\sqrt{d-1}\beta_{t-1,\delta}}{\sqrt{2}\sqrt{\lambda_{\min,t-1}}}.
\end{align}
Thus, using that  $\frac{9\sqrt{3}}{\sqrt 2} \leq 12$ we see that with our definitions of estimator~\eqref{eq:estimator_weighted} and weights 
\begin{align}\label{eq:estimator_choice}
    \hat{\sigma}^2_t (a^\pm_{t,i} ) = \frac{12\sqrt{d-1}\beta_{t-1,\delta}}{\sqrt{\lambda_{\max,t-1}}}, \quad \omega (V_{t-1} ) = \frac{1}{\hat{\sigma}^2_t },
\end{align}
are well defined since only depends on the history $\mathcal{H}_{t-1}$, and 
\begin{align}\label{eq:upperbound_subgaussian}
    \sigma^2_t ( a^\pm_{t,i}) \leq \hat{\sigma}^2_t (a^\pm_{t,i} ) \quad \text{if} \quad \theta\in\mathcal{C}_{t-1}.
\end{align}

To bound the regret our technique uses upper and lower bounds on the scaling of $\Tr (V_t )$ as a function of the number of rounds. In the standard \textsf{LinUCB} technique if the actions are bounded ($\| a_t \|_2 = \Theta (1 )$) and this immediately gives the linear scaling $\Tr (V_t) = \Theta (t)$. Since we are updating $V_t$ using $\omega (V_t )$ we need to do some extra work, and we will see that $\Tr (V_t ) = \tilde{\Theta} (t^2 )$.


\subsection*{Upper bound for $\Tr (V_t )$}

The reason why we need an upper bound for $\Tr (V_t )$ is because we will need an upper bound for $\beta_{t,\delta}$ that depends on $\Tr (V_t )$. A direct calculation shows
\begin{align}
    \Tr (V_t ) &= \lambda_0 + \sum_{s=1}^t 2(d-1)\omega (V_{s-1} ) \\
    &= \lambda_0 + \sum_{s=1}^t \frac{2(d-1)}{12\sqrt{d-1}\beta_{s-1,\delta}}\sqrt{\lambda_{\max , s-1}} \\
    &\leq \lambda_0 + \frac{\sqrt{d-1}}{6\beta_{0.\delta}}\sum_{s=0}^{t-1} \sqrt{\lambda_{\max , s}},
\end{align}
where we used $\beta_{0,\delta} \leq \beta_{t,\delta}$. Then using  $\lambda_{\max , t} \leq \Tr (V_t)$ and $\beta_{0,\delta}\geq 1$ we have
\begin{align}\label{eq:lmaxinequality}
    \lambda_{\max,t} \leq \lambda_0+ \frac{\sqrt{d}}{6}\sum_{s=0}^{t-1} \sqrt{\lambda_{\max , s}}.
\end{align}
Now we want to extend $\lambda_{\max,t}$ to a monotonic function on the interval $[0, \widetilde{T} ]$. In order to do that we define the linear interpolation $g_1:[0, \widetilde{T} ] \rightarrow \mathbb{R}_{\geq 0}$ as
\begin{align}\label{eq:def_gx}
    g_1 (x) := (t+1-x)\lambda_{\max,t} + (x-t)\lambda_{\max,t+1} \quad \text{for} \quad x\in[t,t+1) \quad \text{and} \quad t\in\lbrace 0 ,..., \widetilde{T} \rbrace.
\end{align}
Then using~\eqref{eq:lmaxinequality} and $g_1(t) = \lambda_{\max,t}$ we have
\begin{align}
    g_1(t) \leq \lambda_0 + \frac{\sqrt{d}}{6}\sum_{s=0}^{t-1} \sqrt{g_1(s)} \leq \lambda_0 + \frac{\sqrt{d}}{6}\int_{0}^t \sqrt{g_1 (x)}\dd x \quad \text{for} \quad t\in \lbrace 0 ,..., \widetilde{T}  \rbrace.
\end{align}
Now we want to prove the above inequality but for $g_1 (x)$ with $x\in[0,\widetilde{T} ]$. From the above inequality and the definition of $g_1 (x)$~\eqref{eq:def_gx} we have that for $x\in[t,t+1)$ 
\begin{align}
    g_1 (x) &\leq \lambda_0 + \frac{\sqrt{d}}{6}\sum_{s=0}^{t-1} \sqrt{g_1(s)} + \frac{\sqrt{d}}{6}(x-t)\sqrt{g_1 (t)} \\
    &\leq \lambda_0 + \frac{\sqrt{d}}{6}\int_{0}^t \sqrt{g_1 (x)}\dd x + \frac{\sqrt{d}}{6}(x-t)\sqrt{g_1 (t)}.
\end{align}
Then we can use the above relation and 
\begin{align}
    \int_t^x \sqrt{g_1(s)}\dd s &= \int_0^{x-t} \sqrt{g_1(s+t)}\dd s = \int_0^{x-t} \sqrt{(1-s)g_1(t) + sg_1(t+1)} \dd s \\
    &\geq \int_{0}^{x-t}\sqrt{g_1 (t)} \geq (x-t)\sqrt{g_1(t)},
\end{align}
where the second equality follows from the definition of $g_1(x)$~\eqref{eq:def_gx} and the first inequality $g_1(t) =\lambda_{\max,t}\leq\lambda_{\max,t+1} = g_1(t+1)$. We conclude that
\begin{align}
    g_1 (x) \leq \lambda_0 +  \frac{\sqrt{d}}{6}\int_{0}^{x} \sqrt{g_1(s)}\dd s .
\end{align}
Using the fundamental theorem of calculus
\begin{align}\label{eq:gxupperbound}
    g_1 (x) \leq \lambda_0 + \frac{\sqrt{d}}{6}(G_1 (x)-G_1(0))\quad\text{where} \quad \frac{\dd G_1}{\dd x} = \sqrt{g_1 (x)}.
\end{align}
Fixing $ G(0) = 0$ and rearranging 
\begin{align}
    \frac{\dd G_1}{\dd x} \leq \sqrt{\lambda_0 +\frac{\sqrt{d}}{6} G_1(x) }.
\end{align}
Now solving for the equality with $G(0) = 0 $ and using Theorem~\ref{th:dif_equation_ineq} we arrive at
\begin{align}
    G_1(x) \leq  \frac{\sqrt{d}}{24}x^2 + \sqrt{\lambda_0}x.
\end{align}
Finally inserting the above into~\eqref{eq:gxupperbound} using that $g_1(t) = \lambda_{\max,t}$
\begin{align}
    \lambda_{\max,t} = g_1(t) \leq \frac{d}{144}t^2 + \frac{\sqrt{d\lambda_0}}{6}t + \lambda_0 ,
\end{align}
and this gives the following bound on $\Tr (V_t)$
\begin{align}\label{eq:vtupperbound}
    \Tr ( V_t ) \leq d \lambda_{\max,t}  \leq d\left( \frac{d}{144}t^2 + \frac{\sqrt{d\lambda_0}}{6}t + \lambda_0 \right).
\end{align}


\subsection*{Upper bound for $\beta_{t,\delta}$}

In order to give an upper bound for $\beta_{t,\delta}$~\eqref{eq:beta} we have to give an upper bound for $\det (V_t)$. The inequality of arithmetic and geometric means gives
\begin{align}
    \det (V_t ) \leq \left(\frac{\Tr (V_t )}{d} \right)^d,
\end{align}
which combined with~\eqref{eq:vtupperbound} and~\eqref{eq:beta} gives
\begin{align}\label{eq:beta_upperbound}
    \beta_{t,\delta} \leq \left(\lambda_0+\sqrt{2\log\frac{1}{\delta} + d\log\left(\frac{d}{144\lambda_0}t^2 + \frac{\sqrt{d}}{\sqrt{\lambda_0}6}t +1\right)} \right)^2 .
\end{align}


\subsection*{Lower bound for $\lambda_{\max,t}$}

In order to give an upper bound for regret our main technique is to control the growth of $\lambda_{\max,t}$ through a lower bound for $\lambda_{\max,t}$. We will employ the same technique as in the previous part. For the lower bound we are only interested in the leading terms so  
we start by lower bounding $\Tr ( V_t )$ as
\begin{align}
    \Tr (V_t ) &\geq \sum_{s= 2}^{t} 2(d-1)\omega (V_{s-1} ) \\
    &= \sum_{s=2}^t \frac{\sqrt{d-1}\sqrt{\lambda_{\max,s-1}}}{6\beta_{s-1,\delta}} \\
    & \geq \frac{\sqrt{d-1}}{6\beta_{\widetilde{T},\delta}} \sum_{s=1}^{t-1} \sqrt{\lambda_{\max,s}}. 
\end{align}
where we used $\beta_{t,\delta} \leq \beta_{\widetilde{T},\delta}$.
 Then since our problem is restricted to $\mathbb{R}^d$ we have that for the maximum eigenvalue of $V_t$,
\begin{align}
    \lambda_{\max,t} \geq \frac{\Tr (V_t )}{d}.
\end{align}
Thus combining with the above expressions we have
\begin{align}\label{eq:uplambdmax}
    \lambda_{\max,t} \geq \frac{\sqrt{d-1}}{6d\beta_{\widetilde{T},\delta}} \sum_{s= 1}^{t-1} \sqrt{\lambda_{\max,s}}.
\end{align}
Addding each side by $\frac{\sqrt{d-1}}{6d\beta_{\widetilde{T},\delta}}\sqrt{\lambda_{\max,t}}$ and using the fact that $\lambda_{\max,t} \geq \sqrt{\lambda_{\max,t}}$ since $\lambda_{\max,t}\geq 1$  we have 
\begin{align}\label{eq:lambdmaxinequality}
    \lambda_{\max,t} + \frac{\sqrt{d-1}}{6d\beta_{\widetilde{T},\delta}}\sqrt{\lambda_{\max,t}} &\geq \frac{\sqrt{d-1}}{6d\beta_{\widetilde{T},\delta}} \sum_{s= 1}^{t-1} \sqrt{\lambda_{\max,s}} + \frac{\sqrt{d-1}}{6d\beta_{\widetilde{T},\delta}} \sqrt{\lambda_{\max,t}},\\
    \left(1+ \frac{\sqrt{d-1}}{6d\beta_{\widetilde{T},\delta}} \right)\lambda_{\max,t} &\geq \lambda_{\max,t} + \frac{\sqrt{d-1}}{6d\beta_{\widetilde{T},\delta}}\sqrt{\lambda_{\max,t}} \geq \frac{\sqrt{d-1}}{6d\beta_{\widetilde{T},\delta}} \sum_{s= 1}^{t} \sqrt{\lambda_{\max,s}}, \\
     \lambda_{\max,t} &\geq \frac{1}{1+6\frac{d}{\sqrt{d-1}}\beta_{\widetilde{T},\delta}} \sum_{s= 1}^{t} \sqrt{\lambda_{\max,s}}.
\end{align}
As before we want to extend $\lambda_{\max,t}$ to a monotonic function on the interval $[0, \widetilde{T}]$ and define the lower linear interpolation $g_2:[0, \widetilde{T}] \rightarrow \mathbb{R}_{\geq 0}$ as
\begin{align}\label{eq:def_fx}
    g_2(x) := (t+1-x)\lambda_{\max,t} + (x-t)\lambda_{\max,t+1} \quad \text{for} \quad x\in[t,t+1) \quad \text{and} \quad t\in \lbrace 0 , ..., \widetilde{T}\rbrace .
\end{align}
Then we have that 
\begin{align}
    g_2(t)  &\geq \frac{1}{1+6\frac{d}{\sqrt{d-1}}\beta_{\widetilde{T},\delta}} \sum_{s=1}^{t} \sqrt{g_2(s)}  \\
    &\geq \frac{1}{1+6\frac{d}{\sqrt{d-1}}\beta_{\widetilde{T},\delta}} \int_{0}^t \sqrt{g_2(x)} \dd x \quad \text{for} \quad t\in\lbrace 0 , ... , \widetilde{T} \rbrace,
\end{align}
since $g_2(t) = \lambda_{\max,t}$ for $t\in \lbrace 0 , .... , \widetilde{T} \rbrace$. Then we can check that the above inequality also holds for $x\in [0, \widetilde{T}]$. If $x\in [t,t+1)$ we have
\begin{align}\label{eq:fxbound}
    g_2(x) &\geq \frac{1}{1+6\frac{d}{\sqrt{d-1}}\beta_{\widetilde{T},\delta}} \left( \sum_{s=1}^{t} \sqrt{g_2 (s)} + (x-t)\sqrt{g_2 (t+1) }\right) \\
    & \geq \frac{1}{1+6\frac{d}{\sqrt{d-1}}\beta_{\widetilde{T},\delta}} \left(\int_0^{t}\sqrt{g_2 (x)}\dd x + (x-t)\sqrt{g_2 (t+1)}\right),
\end{align}
where the first inequality we applied the definition of $g_2(x)$~\eqref{eq:def_fx} and the inequality for $\lambda_{\max,t}$~\eqref{eq:uplambdmax}. Then,
\begin{align}
    \int_{t}^{x}\sqrt{g_2(s)}\dd s &= \int_{0}^{x-t}\sqrt{g_2(s+t)}\dd s = \int_{0}^{x-t} \sqrt{(1-s)g_2(t) + sg_2(t+1)}\dd s \\
    &\leq  \int_{0}^{x-t} \sqrt{g_2 (t+1)}\dd s \leq (x-t)\sqrt{g_2(t+1)},
\end{align}
where the second equality follows from the definition of $g_2 (x)$ and the first inequality from $g_2(t) = \lambda_{\max,t} \leq \lambda_{\max , t+1} \leq g_2 (t+1)$. Combining with~\eqref{eq:fxbound} we have
\begin{align}
    g_2(x) &\geq \frac{1}{1+6\frac{d}{\sqrt{d-1}}\beta_{\widetilde{T},\delta}} \int_{0}^x \sqrt{g_2 (s)}ds.
\end{align}
Using the fundamental theorem of calculus 
\begin{align}\label{eq:bound_lambdamax}
    g_2(x) \geq \frac{1}{1+6\frac{d}{\sqrt{d-1}}\beta_{\widetilde{T},\delta}} \left(G_2(x) - G_2(0) \right) \quad \text{where} \quad \frac{\dd G_2(x)}{\dd x} = \sqrt{g_2(x) }.
\end{align}
Fixing $ G_2 (0) = 0$ and rearranging 
\begin{align}
    \frac{\dd G_2(x)}{\dd x} \geq \sqrt{\frac{1}{1+6\frac{d}{\sqrt{d-1}}\beta_{\widetilde{T},\delta}}}\sqrt{G_2(x)}.
\end{align}
Now solving the equality with $G_2 (0) = 0$ and using Theorem~\ref{th:dif_equation_ineq}
\begin{align}
    G_2 (x) &\geq \frac{1}{4+24\frac{d}{\sqrt{d-1}}\beta_{\widetilde{T},\delta}}x^2 .
\end{align}
Using that $g_2(t) = \lambda_{\max ,t} $ and inserting the above result into~\eqref{eq:bound_lambdamax}
\begin{align}\label{eq:lambdamaxlowbound}
   \lambda_{\max,t} = g_2 (t) \geq \frac{1}{4(1+6\frac{d}{\sqrt{d-1}}\beta_{\widetilde{T},\delta})^2} t^2 .
\end{align}
%


\subsection*{Bounding the regret}

From the weight update~\eqref{eq:omega} we have
\begin{align}
    \omega ( V_t ) = \frac{\sqrt{\lambda_{\max,t-1}}}{12\sqrt{d-1}\beta_{t,\delta}} \leq \frac{1}{12\sqrt{d-1}}\sqrt{\lambda_{\max,t-1}}, 
\end{align}
so to use Therorem~\ref{th:main} we can choose $C = \frac{1}{12\sqrt{d-1}}$, our choice of $\lambda_0$~\eqref{eq:lambda_def}, actions~\eqref{eq:general_update} and the sequence $\lbrace \theta^\text{w}_t \rbrace_{t=0}^\infty \subset \mathbb{S}^d$. Then our update given by~\eqref{eq:vt_update} satisfies all conditions of Theorem~\ref{th:main} and our choice of actions guarantees
\begin{align}\label{eq:eigenvalue_relation}
    \lambda_{\min,t} \geq \sqrt{\frac{2}{3(d-1)}\lambda_{\max,t}}.
\end{align}
Then from~\eqref{eq:beta_upperbound} we can upper bound $\beta_{t,\delta}$ as
\begin{align}
 \beta_{t,\delta}\leq \beta_{T,\max} \quad \text{where} \quad  \beta_{T,\max}  := \left(\lambda_0+\sqrt{2\log\frac{1}{\delta} + d\log\left(\frac{d}{144\lambda_0}T^2 + \frac{\sqrt{d}}{\sqrt{\lambda_0}6}T +1\right)} \right)^2 .
\end{align}
Finally, we can use the bound from~\eqref{eq:noise_bound}\eqref{eq:theta_at_bound} (under the assumptions that $E_t$ and $G_t$ hold) combined with~\eqref{eq:eigenvalue_relation} to arrive at
\begin{align}
    \| \theta - a^{\pm}_{t,i} \|_2^2 &\leq \frac{9\beta_{T,\max}}{\lambda_{\min,{t-1}}} \\
    &\leq \frac{12\sqrt{d-1}\beta_{T,\max}} {\sqrt{\lambda_{\max,t-1}}}\\
    &\leq \frac{24\sqrt{d-1}\beta_{T,\max}(1+6\frac{d}{\sqrt{d-1}}\beta_{T,\max})}{t-1} \\
    &= \frac{144d\beta^2_{T,\max}+24\sqrt{d-1}\beta_{T,\max}}{t-1}
\end{align}
where we used the lower bound for $\lambda_{\max,t}$~\eqref{eq:lambdamaxlowbound} and $\frac{9\sqrt{3}}{\sqrt 2} \leq 12$.
Then using the above result we can bound the regret as
\begin{align}
    \text{Regret}(T) &=   \frac{1}{2}\sum_{t= 1}^{\widetilde{T}} \sum_{i = 1}^{d-1} \left( \| \theta - a^+_{t,i} \|_2^2 + \| \theta - a^{-}_{t,i} \|_2^2  \right) \\
    &\leq 4(d-1) + \frac{1}{2}\sum_{t= 2}^{\widetilde{T}} \sum_{i = 1}^{d-1} \left( \| \theta - a^+_{t,i} \|_2^2 + \| \theta - a^{-}_{t,i} \|_2^2  \right) \\
    &\leq 4(d-1) + \left(144d^2\beta^2_{T,\max}+24(d-1)^{\frac{3}{2}}\beta_{T,\max}\right)\sum_{t= 2}^{\tilde{t}} \frac{1}{t-1}  \\
    & \leq 4(d-1) + \left(144d^2\beta^2_{T,\max}+24(d-1)^{\frac{3}{2}}\beta_{T,\max}\right)\log \left( \frac{T}{2(d-1)} \right).
\end{align}
Using that $\beta_{T,\max} = O (d\log(T))$ we have that 
\begin{align}
    \text{Regret}(T) = \widetilde{O}(d^{4}\log^3 (T)).
\end{align}
%


\subsection*{Success probability analysis}

From the computations in~\eqref{eq:noise_bound}\eqref{eq:theta_at_bound}\eqref{eq:estimator_choice}\eqref{eq:upperbound_subgaussian} and
 our choice of $\sigma^2_t(a_t)$~\eqref{eq:omega} we have that 
\begin{align}
    \text{if} \quad \theta\in\mathcal{C}_{s-1} \Rightarrow \sigma^2_s (a^{\pm}_{s,i} ) \leq \frac{12\sqrt{d-1}\beta_{t-1,\delta}}{\sqrt{\lambda_{\max,t-1}}} = \hat{\sigma}^2_s (a^{\pm}_{s,i}), 
\end{align}
thus,
\begin{align}
     \mathrm{Pr}( G_t ) \geq \mathrm{Pr}(E_{t-1} ) 
\end{align}
where we have used the definitions of the events $G_t$~\eqref{eq:event_Gt},$E_t$~\eqref{eq:event_Et}.

The initial choice $\hat{\sigma}^2_1 = 1$ implies $\sigma^2\leq \hat{\sigma}^2_1 $ and using Lemma~\ref{lem:confidence_region_weighted} we have
\begin{align}\label{eq:probg1}
     \mathrm{Pr}(G_1 ) & = 1 , \\
     \mathrm{Pr}(E_1 ) &  \geq 1-\delta .
\end{align}
Using Bayes theorem
\begin{align}
    \mathrm{Pr}(E_t )  = \frac{\mathrm{Pr}(E_t |G_t)\mathrm{Pr}(G_t)}{\mathrm{Pr}(G_t |E_t)} = \mathrm{Pr}(E_t |G_t)\mathrm{Pr}(G_t),
\end{align}
where in the last equality we used
\begin{align}
    \mathrm{Pr}(G_t |E_t) = 1 .
\end{align}
From Lemma~\ref{lem:confidence_region_weighted} we have
\begin{align}
    \mathrm{Pr} (E_t | G_t ) \geq 1 -\delta . 
\end{align}
Thus, applying recursively the above inequalities 
\begin{align}\label{eq:probgt}
     \mathrm{Pr}(G_2 ) & \geq \mathrm{Pr}(E_1 ) \geq 1-\delta,\\
     \mathrm{Pr}(E_2 ) & = \mathrm{Pr}(E_2 | G_2)\mathrm{Pr}(G_2)\geq (1- \delta)^2 \\
      &\vdots \nonumber \\
      \mathrm{Pr}(G_{t} ) & \geq \mathrm{Pr}(E_{t-1} )  \geq (1-\delta)^{t-1} \\
      \mathrm{Pr}(E_t ) & = \mathrm{Pr}(E_t | G_t)\mathrm{Pr}(G_t)\geq (1- \delta)^{t}     
\end{align}
Finally to bound the regret we used the assumption that $\theta \in\mathcal{C}_t$ for any $t\in \lbrace 1,...,\widetilde{T}-1\rbrace$. Thus we can bound the probability that the obtained bound holds as
\begin{align}
     \mathrm{Pr}&\left( \text{Regret}(T) \leq  4(d-1) + \left(144d^2\beta^2_{T,\max}+24(d-1)^{\frac{3}{2}}\beta_{T,\max}\right)\log\left( \frac{T}{2(d-1)} \right)\right) \\ 
    & \geq \mathrm{Pr}(E_{\widetilde{T}} \cap G_{\widetilde{T}}) = \mathrm{Pr}(G_{\widetilde{T}} )\mathrm{Pr}(E_{\widetilde{T}} | G_{\widetilde{T}} ). \\
    &\geq (1-\delta)^{\widetilde{T}-1}(1-\delta) = (1-\delta)^{\widetilde{T}}.
\end{align}
The result follows choosing $\delta = \frac{\delta '}{\widetilde{T}}$ for some $0 < \delta' < 1$, the inequality $\left( 1 -  \frac{\delta '}{\widetilde{T}} \right)^{\widetilde{T}} \geq 1-\delta'$ and $\frac{\sqrt{d}}{d-1} \leq 2$.
\end{proof}

\begin{corollary}\label{cor:expected_regret_linucbvn}
Under the same assumptions of Theorem~\ref{th:regret_bound_d2}, we can choose $\delta = \frac{1}{\widetilde{T}}$ and Algorithm~\ref{alg:weighted_linUCB} achieves
\begin{align}
    \EX[\textup{Regret}(T)] \leq 8(d-1) + \left(144d^2\beta^2_{T,\max}+24(d-1)^{\frac{3}{2}}\beta_{T,\max}\right)\log\left( \frac{T}{2(d-1)} \right),
\end{align}
or 
\begin{align}
    \EX [\textup{Regret}(T)] = O( d^{4}\log^3 (T) )  .
\end{align}

\end{corollary}

\begin{proof}   
    Defining the event
    \begin{align}
        R_T = \bigg\lbrace  \mathcal{H}_{\widetilde{T}}: 
        \text{Regret}(T) \leq 4(d-1) + \left(144d^2\beta^2_{T,\max}+24(d-1)^{\frac{3}{2}}\beta_{T,\max}\right)\log\left( \frac{T}{2(d-1)} \right) \bigg\rbrace
    \end{align}
    we have by Theorem~\ref{th:regret_bound_d2} that
    \begin{align}
        \mathrm{Pr}(R_T) &\geq 1 - \delta = 1 - \frac{1}{\widetilde{T}}, \\
        \mathrm{Pr}(R_T^C) &\leq \frac{1}{\widetilde{T}}.
    \end{align}
    Then
    \begin{align}
            \EX[\text{Regret}(T)] &= \EX\left[\text{Regret}(T) \mathbb{I}\lbrace R_T \rbrace\right] + \EX\left[\text{Regret}(T) \mathbb{I}\lbrace R_T^C \rbrace \right] \nonumber \\
            &\leq 4(d-1) + \left(144d^2\beta^2_{T,\max}+24(d-1)^{\frac{3}{2}}\beta_{T,\max}\right)\log\left( \frac{T}{2(d-1)} \right) + 4(d-1)\widetilde{T} \mathrm{Pr}\left( R_T^C \right) \\
            &\leq 8(d-1) + \left(144d^2\beta^2_{T,\max}+24(d-1)^{\frac{3}{2}}\beta_{T,\max}\right)\log\left( \frac{T}{2(d-1)} \right),
    \end{align}
    and the results follows using that for $\delta = \frac{1}{\widetilde{T}}$, $\beta_{T,\max} = O (d\log (T))$.   
\end{proof}

\section{Minimax lower bound for linear bandit $\mathcal{A} = \mathbb{S}^d, \theta\in\mathbb{S}^d$ and constant noise}\label{ap:lowerbound}

In this section, we prove that any small perturbation in our noise model leads to $\Omega(\sqrt{T})$ regret. We study a minimax lower bound for the following reward model
\begin{align}\label{eq:noisy_classical_quantum_reward2}
    r_t =  \langle a_t,\theta\rangle +  \mathcal{N}\left(0, \sigma^2_{t,\theta} \right) + \mathcal{N}\left(\tilde{\mu}, \tilde{\sigma}^2 \right) = \mathcal{N}(\tilde{\mu}_{t,\theta}, \tilde{\sigma}^2_{t,\theta} ),
\end{align}
where $\theta\in\mathbb{S}^d$, $\tilde{\mu},\tilde{\sigma}^2\in\mathbb{R}$, $\tilde{\sigma}^2 > 0$,  and 
\begin{align}\label{eq:def_variances_noisy}
    \sigma^2_{t,\theta} := 1 - \langle a_t ,\theta \rangle^2,\quad \tilde{\mu}_{t,\theta}:= \langle a_t,\theta\rangle + \tilde{\mu}_{t,\theta}, \quad \tilde{\sigma}^2_{t,\theta} := \sigma^2_{t,\theta} + \tilde{\sigma}^2.
\end{align}
Recall that the regret is given by
\begin{align}\label{eq:logisticregret}
    \text{Regret} (T) = \sum_{t=1}^T 1 - \langle a_t,\theta\rangle.
\end{align}
Our lower bound proof is an adaptation of the lower bound given in~\cite{abeille2021instance} that was introduced for logistic bandits and this provides a lower bound for linear bandits with $\mathcal{A} =  \mathbb{S}^d,\theta\in \mathbb{S}^d$. For completeness, in  Lemma~\ref{lem:regret_logistic_bandits} we reproduce the main steps and we note that in our setting some parts simplify.  In order to state the result we define $\lbrace e_i \rbrace_{i=1}^d $ the standard basis in $\mathbb{R}^d$ and the flip operator. Given $\theta\in\mathbb{R}$, $i\in [d]$ the flip operator is defined as
\begin{align}
    \text{Flip}_i (\theta ) := (\theta_1,\theta_2,...,-\theta_i,...,\theta_d ).
\end{align}
\begin{lemma}\label{lem:regret_logistic_bandits}
Given a stochastic linear bandit with action set $\mathcal{A} =  \mathbb{S}^d = \lbrace x\in\mathbb{R}^d : \| x \|_2 =1 \rbrace$, a policy $\pi$, a reference parameter $\theta_* = \|\theta_* \|_2e_1\in\mathbb{R}^d$ and the set of parameters
\begin{align}
    \Xi = \left\lbrace \theta_* + \epsilon\sum_{i=2}^d v_ie_i , \quad v_i \in \lbrace -1 , 1 \rbrace \right\rbrace,
\end{align}
where $0< \epsilon\leq \frac{1}{d-1}$.
Then the average of the expected regret over the set can be lower bounded as
\begin{align}
   \frac{1}{|\Xi |} \sum_{\theta\in\Xi} \EX_\theta\left[ \textup{Regret}(T) \right] \geq  T \epsilon^2  \left(\frac{d}{8} - \frac{\sqrt{d}}{4}\sqrt{\frac{1}{|\Xi |}\sum_{\theta\in\Xi}\sum_{i=2}^d D_{KL} \left( \mathbb{P}_\theta , \mathbb{P}_{\textup{Flip}_i (\theta)} \right)} \right),
\end{align}
where $\mathbb{P}_\theta,\mathbb{P}_{\textup{Flip}_i (\theta)}$ are the probability measures of actions and rewards obtained by the interaction of the policy $\pi$ with the environments $\theta$,$\textup{Flip}_i (\theta)$ respectively. Moreover, $\epsilon$ can be fixed such that for all $\theta\in\Xi$ then $\theta\in\mathcal{B}^d_2$.
\end{lemma}

\begin{proof}
First, we note that for any $\theta\in\Xi$ it has a constant norm, i.e 
\begin{align}
    \| \theta \| = \sqrt{\| \theta_* \|^2 + (d-1)\epsilon^2}, \quad \text{for }\theta\in\Xi .
\end{align}
Thus, we fix 
\begin{align}
    \| \theta_* \| = \sqrt{1-(d-1)\epsilon^2}
\end{align}
and we have that 
\begin{align}
    \text{if} \quad \theta \in \Xi \Rightarrow \theta \in\mathcal{S}_{d-1}.
\end{align}
Note that when choosing $\epsilon$ we will have the following restriction
\begin{align}
    \epsilon^2 \leq \frac{1}{d-1}.
\end{align}
From the expression of the regret we have
\begin{align}
    \EX_\theta \left[ \text{Regret}_\theta (T)\right] &= \frac{1}{2}\EX_{\theta}\left[ \sum_{t=1}^T \| \theta - a_t \|^2 \right] = \frac{1}{2}\EX_{\theta}\left[ \sum_{i=1}^d \sum_{t=1}^T \left[ \theta - a_t \right]_i^2 \right] \\ &\geq \frac{1}{2}\EX_{\theta}\left[ \sum_{i=1}^d \sum_{t=1}^T \left[ \theta - a_t \right]_i^2 \mathbbm{1}\lbrace A_i(\theta) \rbrace \right] = (\text{a})
\end{align}
where the event $A_i(\theta)$ is defined as
\begin{align}
    A_i(\theta) := \left\lbrace \left[\theta -  \frac{\theta_*}{\| \theta_* \|} \right]_i \left[ \sum_{t=1}^T\frac{\theta_*}{\| \theta_* \|} - a_t\right]_i \geq 0 \right\rbrace,
\end{align}
for $i=1,2,...,d$. Then we want to compare with the reference environment $\theta_*$, and we can introduce it as
\begin{align}
    (\text{a}) &= \frac{1}{2}\EX_{\theta}\left[ \sum_{i=1}^d \sum_{t=1}^T \left[ \theta - \frac{\theta_*}{\|\theta_*\|}+ \frac{\theta_*}{\|\theta_*\|} - a_t \right]_i^2 \mathbbm{1}\lbrace A_i(\theta) \rbrace \right] \\ &\geq \frac{1}{2}\EX_{\theta}\left[ \sum_{i=1}^d \sum_{t=1}^T \left[ \theta - \frac{\theta_*}{\|\theta_*\|} \right]_i^2 \mathbbm{1}\lbrace A_i(\theta) \rbrace \right] = \frac{T}{2}\EX_{\theta}\left[ \sum_{i=1}^d \left[ \theta - \frac{\theta_*}{\|\theta_*\|} \right]_i^2 \mathbbm{1}\lbrace A_i(\theta) \rbrace \right] 
\end{align}
where the last inequality follows from the identity $(x+y)^2 = x^2+2xy+y^2$ with $x = \theta - \frac{\theta_*}{\| \theta_* \|} $, $y = \frac{\theta_*}{\|\theta_*\|} - a_t $ and the definition of the event $A_i(\theta)$. Using the above computation and that $\theta\in\Xi$,
\begin{align}\label{eq:regret_with_prob}
    \EX_\theta \left[ \text{Regret}_\theta (T)\right] &\geq  \frac{T}{2}\EX_{\theta}\left[ \sum_{i=2}^d \left[ \theta - \frac{\theta_*}{\|\theta_*\|} \right]_i^2 \mathbbm{1}\lbrace A_i(\theta) \rbrace \right] \\
    &\geq \frac{T\epsilon^2}{2} \sum_{i=2}^d \EX_{\theta}\left[ \mathbbm{1}\lbrace A_i(\theta) \rbrace \right] = \frac{T\epsilon^2}{2} \sum_{i=2}^d \mathbb{P}_\theta \left( A_i(\theta) \right). 
\end{align}
Now we want to apply the averaging hammer technique that was introduced in~\cite[Chapter 24, Theorem 24.1]{lattimore_szepesvári_2020}. Let $i\in \lbrace 2,...,d \rbrace$, fix $\theta\in\Xi$ then using that $[ \theta_* ]_i = 0$ it is easy to check that
\begin{align}
    A_i (\text{Flip}_i (\theta )) = A^C_i (\theta).
\end{align}
Then using the definition of the total variation distance and Pinsker inequality we have
\begin{align}
    \mathbb{P}_{\text{Flip}_i(\theta )} (A_i (\text{Flip}_i(\theta) ) &\geq \mathbb{P}_\theta ( A_i (\text{Flip}_i ( \theta )  ) ) - D_{TV} (\mathbb{P}_\theta ,\mathbb{P}_{\text{Flip}_i(\theta )}) \\
    &\geq \mathbb{P}_\theta ( A^C_i (\theta)) - \sqrt{\frac{1}{2}D_{KL}(\mathbb{P}_\theta ,\mathbb{P}_{\text{Flip}_i(\theta )})}. 
\end{align}
Then applying the above results we have
\begin{align}
  \frac{1}{|\Xi |} \sum_{\theta\in\Xi} \sum_{i=2}^d \mathbb{P}_\theta \left( A_i(\theta) \right) &\geq \frac{1}{2|\Xi |}\sum_{i=2}^d \sum_{\theta\in\Xi} \left( \mathbb{P}_{\theta}(A_i (\theta ) )  +\mathbb{P}_{\text{Flip}_i(\theta )} (A_i (\text{Flip}_i(\theta) )\right) \\
  & \geq \frac{1}{2|\Xi |}\sum_{i=2}^d \sum_{\theta\in\Xi} 1 - \sqrt{\frac{1}{2}D_{KL}(\mathbb{P}_\theta ,\mathbb{P}_{\text{Flip}_i(\theta )})}.
\end{align}
Using Jensen inequality, Cauchy-Schwartz, and the fact that $d\geq1$ we have
\begin{align}\label{eq:sum_prob}
    \frac{1}{|\Xi |} \sum_{\theta\in\Xi} \sum_{i=2}^d \mathbb{P}_\theta \left( A_i(\theta) \right) &\geq \frac{d}{4} - \frac{\sqrt{d}}{2}\sqrt{\frac{1}{|\Xi |}\sum_{\theta\in\Xi}\sum_{i=2}^d D_{KL} \left( \mathbb{P}_\theta , \mathbb{P}_{\text{Flip}_i (\theta)} \right)}
\end{align}
And the result follows from combining~\eqref{eq:regret_with_prob} and~\eqref{eq:sum_prob}.
\end{proof}

Before proving the main theorem we need a formula for the KL divergence between two normal distributions. Given $\mathcal{N}(\mu_1,\sigma^2_1)$ and $\mathcal{N}(\mu_2,\sigma^2_2)$ it follows from a direct calculation that
\begin{align}\label{eq:kl_gaussians}
    D_{KL} \left(\mathcal{N}(\mu_1,\sigma^2_1),\mathcal{N}(\mu_2,\sigma^2_2) \right) = \frac{1}{2}\left( \log\left( \frac{\sigma^2_2}{\sigma^2_1}\right)+ \frac{\sigma^2_1}{\sigma^2_2}-1 \right) + \frac{(\mu_1-\mu_2)^2}{\sigma_2^2}.
\end{align}
With these results, we are ready to prove the main theorem.
\begin{theorem}
Given a stochastic linear bandit with action set $\mathcal{A} =  \mathbb{S}^d = \lbrace x\in\mathbb{R}^d: \| x \|_2 =1 \rbrace$ and reward model given by~\eqref{eq:noisy_classical_quantum_reward2}, then there exists an unknown parameter $\theta\in\mathbb{R}^d$ such that $\|\theta \|_2 = 1$ and
\begin{align}
    \EX_\theta [\textup{Regret} (T)] \geq \frac{1}{100}\hat{\sigma}d\sqrt{T}, 
\end{align}
for $T\geq \frac{1}{6400}d^2\tilde{\sigma}^2$.
\end{theorem}
\begin{proof}
First we suppose that for any $\theta \in \Xi$,
\begin{align}\label{eq:regret_hypothesis}
    \EX_\theta [\textup{Regret} (T)] \leq C d \tilde{\sigma} \sqrt{T},
\end{align}
for some constant $C >0$ that without loss of generality we set to $C=1$. We see from Lemma~\ref{lem:regret_logistic_bandits} that we need to compute $D_{KL} \left( \mathbb{P}_\theta , \mathbb{P}_{\text{Flip}_i (\theta)} \right)$. From~\cite{lattimore_szepesvári_2020}[Chapter 24, Theorem 24.1] we have
\begin{align}
    D_{KL} \left( \mathbb{P}_\theta , \mathbb{P}_{\text{Flip}_i (\theta)} \right) = \EX_\theta \left[\sum_{t=1}^T D_{KL}\left( \mathcal{N}(\tilde{\mu}_{t,\theta},\tilde{\sigma}^2_{t,\theta}),\mathcal{N}(\tilde{\mu}_{t,\text{Flip}_i(\theta)},\tilde{\sigma}^2_{t,\text{Flip}_i(\theta)}) \right) \right].
\end{align}
Using~\eqref{eq:def_variances_noisy} and~\eqref{eq:kl_gaussians} we have
\begin{align}
   &D_{KL}\left( \mathcal{N}(\tilde{\mu}_{t,\theta},\tilde{\sigma}^2_{t,\theta}),\mathcal{N}(\tilde{\mu}_{t,\text{Flip}_i(\theta)},\tilde{\sigma}^2_{t,\theta'}) \right) =\\ &\frac{1}{2}\left(\underbrace{\log\left(\frac{1 - \langle a_t, \text{Flip}_i(\theta)\rangle^2 +\tilde{\sigma}^2 }{1 - \langle a_t,\theta \rangle^2 +\tilde{\sigma}^2}\right) }_{\text{(a)}}+\underbrace{\frac{1 - \langle a_t, \theta\rangle^2 +\tilde{\sigma}^2 }{1 - \langle a_t,\text{Flip}_i(\theta)\rangle^2 +\tilde{\sigma}^2}}_{\text{(b)}} -1\right)  \\
   &+ \frac{\langle a_t,\theta - \text{Flip}_i(\theta)\rangle ^2}{1 - \langle a_t,\text{Flip}_i(\theta)\rangle^2 +\tilde{\sigma}^2}.
\end{align}
Now we are going to upper bound the terms (a) and (b).
\begin{align}
 \text{(a)} &= \log \left( 1 + \frac{\langle a_t,\theta \rangle^2 - \langle a_t, \text{Flip}_i(\theta)\rangle^2}{1 - \langle a_t,\theta \rangle^2 +\tilde{\sigma}^2}\right) \leq  \frac{\langle a_t,\theta \rangle^2 - \langle a_t, \text{Flip}_i(\theta)\rangle^2}{1 - \langle a_t,\theta \rangle^2 +\tilde{\sigma}^2}  \\ &\leq \frac{\langle a_t,\theta \rangle^2 - \langle a_t, \text{Flip}_i(\theta)\rangle^2}{\tilde{\sigma}^2} \\
    \text{(b)} &= 1 + \frac{\langle a_t,\text{Flip}_i(\theta)\rangle^2 - \langle a_t, \theta\rangle^2  }{1 - \langle a_t,\text{Flip}_i(\theta)\rangle^2 +\tilde{\sigma}^2} \leq 1 + \frac{\langle a_t,\text{Flip}_i(\theta)\rangle^2 - \langle a_t, \theta\rangle^2  }{\tilde{\sigma}^2}
\end{align}
Thus, we have that $\text{(a)}+\text{(b)}-1\leq 0$ and
\begin{align}
    D_{KL}\left( \mathcal{N}(\tilde{\mu}_{t,\theta},\tilde{\sigma}^2_{t,\theta}),\mathcal{N}(\tilde{\mu}_{t,\text{Flip}_i(\theta)},\tilde{\sigma}^2_{t,\theta'}) \right) \leq \frac{4\epsilon^2}{\tilde{\sigma}^2} \left[a_t \right]^2_i .
\end{align}
Inserting the above into $D_{KL}\left( \mathbb{P}_\theta , \mathbb{P}_{\text{Flip}_i (\theta)} \right)$ we have
\begin{align}
   &\sum_{i=2}^d D_{KL} \left( \mathbb{P}_\theta , \mathbb{P}_{\text{Flip}_i (\theta)} \right) \leq \frac{4\epsilon^2}{\tilde{\sigma}^2} \sum_{i=2}^d \EX_{\theta}\left[ \sum_{t=1}^T \left[\theta - a_t +\theta \right]^2_i \right]  \\ 
   &\leq \frac{8\epsilon^2}{\tilde{\sigma}^2} \EX_{\theta} \left[\sum_{t=1}^T \sum_{i=2}^d \left[\theta - a_t \right]^2_i + \left[ \theta \right]^2_i \right] 
   \leq \frac{8\epsilon^2}{\tilde{\sigma}^2} \EX_{\theta} \left[\sum_{t=1}^T \sum_{i=1}^d \left[\theta - a_t \right]^2_i + \sum_{t=1}^T\sum_{i=2}^d\left[ \theta \right]^2_i \right] \\
   &= \frac{8\epsilon^2}{\tilde{\sigma}^2}\left(2\EX_\theta[\text{Regret} ( T )] +T(d-1)\epsilon^2\right) \leq \frac{16\epsilon^2\tilde{\sigma}d\sqrt{T}}{\tilde{\sigma}^2} + \frac{8d\epsilon^4 T }{\tilde{\sigma}^2}
\end{align}
where we have used that $d\geq 1$, $\theta\in\Xi$, and~\eqref{eq:regret_hypothesis}. Thus, inserting the above into the result of Lemma~\ref{lem:regret_logistic_bandits},
\begin{align}
    \frac{1}{|\Xi |} \sum_{\theta\in\Xi} \EX_\theta \left[\textup{Regret} (T) \right] \geq  dT \epsilon^2  \left(\frac{1}{8} - \frac{1}{4}\sqrt{\frac{16\tilde{\sigma}\epsilon^2 \sqrt{T}+ 8\epsilon^4 T}{\tilde{\sigma}^2}} \right) 
\end{align}
Finally, choosing $\epsilon^2 = \frac{\tilde{\sigma}}{80\sqrt{T}}$ we have
\begin{align}
    \frac{1}{|\Xi |} \sum_{\theta\in\Xi} \EX_\theta \left[\textup{Regret} (T) \right] &\geq \tilde{\sigma}d\sqrt{T}\left( \frac{1}{8} - \frac{1}{4}\sqrt{\frac{16}{80}+\frac{8}{6400}} \right) \\
    &\geq \frac{1}{100}\tilde{\sigma}d\sqrt{T}.
\end{align}
We note that in order to have $\epsilon^2 \leq \frac{1}{d-1}$ we need $T\geq \frac{1}{6400}d^2\tilde{\sigma}^2$. Note that we proved the result under the hypothesis that~\eqref{eq:regret_hypothesis} holds. If~\eqref{eq:regret_hypothesis} does not hold the result follows trivially. 
\end{proof}

\section{Failure of lower-bound methods for vanishing noise}\label{ap:failurelower}

Specifically, we encounter the problem that these calculations rely on the computation of the $KL$ divergences $D_{KL} (\mathbb{P}_\theta , \mathbb{P}_{\theta '} ) $ for two ``close'' unknown parameters $\theta, \theta ' \in\mathbb{S}^d$ and it is not possible to give a non-trivial bound since their variance is arbitrarily close to $0$ when we select actions near the unknown parameters. We can easily see this fact for example if we consider Gaussian distributed rewards and use the divergence decomposition lemma that leads to
\begin{align}
    D_{KL}(\mathbb{P}_\theta , \mathbb{P}_{\theta '} ) = \EX_\theta \left[ \sum_{t=1}^T D_{KL}\left( \mathcal{N}(\langle \theta , a_t  \rangle, 1 - \langle \theta ,a_t \rangle^2 ) , \mathcal{\mathcal{N}}(\langle \theta' , a_t  \rangle, 1 - \langle \theta' ,a_t \rangle^2) \right) \right],
\end{align}
where $\mathcal{N}(\mu,\sigma^2)$ is a Gaussian probability distribution with mean $\mu\in\mathbb{R}$ and variance $\sigma^2 \geq 0$. The proof requires to upper bound of the above divergence but a simple computation shows that the divergence $ D_{KL}$ for the reward distributions cannot be upper bounded since the variances can go arbitrarily close to $0$ and we only get the trivial bound $D_{KL} \leq \infty$.

\end{document}